\documentclass[11pt, letterpaper]{article}

\usepackage{epsfig,mst-stylefile,amssymb,amsmath, multirow, url, tcolorbox,booktabs, enumitem}
\usepackage{epsfig, amssymb, multirow, url, tcolorbox,booktabs}

\usepackage[ruled,vlined]{algorithm2e}

\usepackage{natbib}
\usepackage[utf8]{inputenc}
\usepackage{bbm}
\SetKwInput{KwInput}{Input}
\usepackage{amsmath}

\newtheorem{@assumption}{\sc Assumption}[section]

\newcommand{\E}{\mathbb{E}}

\newcommand {\argmin}[1]{\underset{#1}{\rm{argmin}}} 
\newcommand {\argmax}[1]{\underset{#1}{\rm{argmax}}}

\newcommand{\pkg}[1]{{\normalfont\fontseries{b}\selectfont #1}}
\let\proglang=\textsf
\let\code=\texttt

\title
{Block Dense Weighted Networks \\with Augmented Degree Correction} 
\author{Benjamin Leinwand \;\;\;\;\;\;\;\;\;  Vladas Pipiras\\
University of North Carolina at Chapel Hill
}

\begin{document}
\maketitle
\begin{abstract}
Dense networks with weighted connections often exhibit a community like structure, where although most nodes are connected to each other, different patterns of edge weights may emerge depending on each node's community membership. We propose a new framework for generating and estimating dense weighted networks with potentially different connectivity patterns across different communities. The proposed model relies on a particular class of functions which map individual node characteristics to the edges connecting those nodes, allowing for flexibility while requiring a small number of parameters relative to the number of edges. By leveraging the estimation techniques, we also develop a bootstrap methodology for generating new networks on the same set of vertices, which may be useful in circumstances where multiple data sets cannot be collected. Performance of these methods are analyzed in theory, simulations, and real data.
\end{abstract}
\textbf{Keywords:} Dense Networks; Weighted Networks; Degree Corrected Block Model; Bootstrap; Community Detection
\section{Introduction}\label{s:intro}
We are interested in modeling dense weighted networks with real, continuous-valued weights $W_{uv}$ for pairs of nodes $u$, $v \in V$, where $V$ denotes the set of all nodes (vertices) in the network. Denseness means that the edge $W_{uv}$ is present for all pairs $u$, $v$ where $u \ne v$, though we will also discuss the case where a proportion of the weights are “missing.” Examples include structural brain networks (as in Figure \ref{f:LDF}) and correlation networks. Focusing on dense weighted networks, what are natural modeling approaches? At the simplest level, if a network has no meaningful structure, one could postulate that 

\begin{equation}  \label{e:noise}
f(W_{uv}) = \epsilon_{uv},
\end{equation}
where $\epsilon_{uv}$ are i.i.d.\ $\cal{N}$$(0,1)$ random variables and $f$ is a function mapping $W_{uv}$ to the (standard) ``normal” space. One can take 
\begin{equation}\label{dist-to-normal}
f(w) = \Phi^{-1}(G(w)),
\end{equation}
where $G$ is the CDF of $W_{uv}$ and $\Phi^{-1}$ is the inverse CDF of $\cal{N}$$(0,1)$. In practice, we can substitute the empirical CDF for the unknown true $G$.  

Most networks of interest, including most real networks, however, have some kind of structure. Two common structures involve communities, broadly understood as sets of particular nodes whose edges exhibit similar connectivity patterns, and degree correction, broadly understood as certain nodes having consistently more edges (or, in dense weighted networks, greater edge weights) than other nodes. A simple way to capture degree correction, which we refer to as ``sociability" (or \textit{SC}, for short) in the ``normal" space is to set 
\begin{align}\label{e:NLSM}
f(W_{uv}) & =   \alpha Z_u + \beta Z_v + \sigma \epsilon_{uv}  \nonumber\\ 
& = \alpha \Phi_1^{-1}(\Psi_u) + \beta \Phi_1^{-1}(\Psi_v) + \sigma \epsilon_{uv} =:  h(\Psi_u, \Psi_v) + \sigma \epsilon_{uv},
\end{align}
where $Z_u$ and $Z_v$ are i.i.d.\ $\cal{N}$$(0,1)$ variables, associated with nodes $u$ and $v$ respectively, and $\alpha, \beta$ and $\sigma$ are other model parameters. In the equivalent expression, $\Psi_u$ are i.i.d.\ U$(0,1)$ random variables (uniform on the interval $(0,1)$), and the function  $h(x, y) =  \alpha \Phi^{-1}(x)+ \beta \Phi^{-1}(y)$ will be allowed to take more general forms below. In the ``normal" space, the function $h(\Phi(z_u), \Phi(z_v)) = \alpha z_u + \beta z_v$ is linear. In the case where $\alpha, \beta >0$, a larger value of $Z_u$ will tend to make $W_{uv}$ values larger across all $v$'s, inducing a sociability structure that reflects degree correction. The term $h(\Psi_u, \Psi_v)$ in (\ref{e:NLSM}) is thus the SC term in the model. The term $ \sigma \epsilon_{uv}$ is thought to consist of independent variables. 

Models of the type (\ref{e:NLSM}) appear in \cite{Fosdick_Hoff_2015}, who also include a multiplicative interaction term. In a departure from that work, we allow for community structures and more general functions $h$ than (\ref{e:NLSM}). After exponentiation, and at the conditional mean level, note also that  (\ref{e:NLSM}) yields

\begin{equation}  \label{e:NLSM-expectation}
\E(e^{f(W_{uv})}|Z) = e^{\alpha Z_u}e^{\beta Z_v}e^{\frac{\sigma^2}{2}} =: \theta_u\theta_v \mu.
\end{equation}   	
Specifications of the form (\ref{e:NLSM-expectation}) are common for connection probabilities in \textit{unweighted} degree corrected or SC models. See degree corrected stochastic block models (DCBMs) in \cite{karrer_newman_2011}, \cite{Gao_Ma_Zhang_Zhou_2018},  or their extensions, popularity adjusted stochastic block models (PABMs) in \cite{sengupta_chen_2017}, \cite{noroozi2019estimation}. While in the DCBM, ``sociability" parameters are global, in the PABM, each node has a possibly different sociability parameter for each community in the network. The models considered here are close in spirit to PABMs and we draw from the techniques in \cite{noroozi2019estimation} to analyze them. However, our focus is on weighted networks where information may be encoded in the patterns of the edge weights rather than the existence of particular edges in a given network. 
We shall thus also consider community versions of the model (\ref{e:NLSM}), where the function $h$ and the parameter $\sigma$ can depend on the pair of communities to which $u$ and $v$ belong. Importantly, according to this definition, communities are not necessarily defined by higher or lower propensities to connect with entire other communities, but rather by particular patterns of edge weights which represent ``preferences" for specific nodes over others  \textit{within the same community}. 

As noted above, we will go beyond the ``linear" sociability patterns encoded by the particular function $h$ shown in (\ref{e:NLSM}) while, perhaps surprisingly, remaining in the ``normal" space. To motivate this extension, instead write the SC term in (\ref{e:NLSM}) as 
\begin{equation}  \label{e:H-expansion}
h(\Psi_u, \Psi_v) = d \Phi^{-1}(H(\Psi_u, \Psi_v)),
\end{equation}  
where $d \in [0,1]$ and $H(x,y) = \Phi(c^{-1}\alpha \Phi^{-1}(x) + c^{-1}\beta \Phi^{-1}(y))$ with $c = \sqrt{\alpha^2 + \beta^2}$. $\alpha$ and $\beta$ modulate the influence of $y$ relative to $x$. The constant $c$ serves a normalizing role so that $H(\Psi_u, \Psi_v)$ is ensured to be U$(0,1)$, and hence the value of $h(\Psi_u, \Psi_v)$ resides in the ``normal" space with variance $d^2$. Plugging (\ref{e:H-expansion}) into the last term of (\ref{e:NLSM}) and constraining $d^2 + \sigma^2 =1$ ensures the resulting values of $f(W_{uv})$ are in the (standard) ``normal" space. 

The critical observation, though, is that plugging (\ref{e:H-expansion}) into (\ref{e:NLSM}) while constraining $d^2 + \sigma^2 =1$ will output values in the (standard) ``normal" space for any function $H$ where $H(\Psi_u, \Psi_v) \sim$ U$(0,1)$, including functions that bear no similarity to the normal distribution as shown beneath (\ref{e:H-expansion}). We shall consider several broad classes of such $H$-functions. Examples of the sociability patterns resulting from various considered $H$-functions are depicted in Figure \ref{f:different-h-funcs} below. The key point is that while $H$ could be associated with quite different SC patterns, the SC term (\ref{e:H-expansion}) would nonetheless reside in the ``normal" space. Used in conjunction with (\ref{dist-to-normal}), which transforms an arbitrary (and possibly nonparametric) distribution of edge weights, this constructs a map between $\Psi$ values and edge weights via the (standard) ``normal" space. In summary, our key contributions at the model level concern: 
\begin{itemize}
  \item Focus on dense weighted networks;
  \item Possibility of community structure;
  \item Multiple nearly arbitrary distributions of edge weights;
   \item Flexible sociability patterns through $H$-functions. 
\end{itemize} 

Modeling questions will also be addressed in the paper below. Figure \ref{f:LDF} illustrates our modeling approach. It shows a network where the edge weights are the logs of the white matter fiber counts connecting two regions in a patient's brain. In this case, the two assumed communities are the left and right hemispheres of the brain. We often reorder the nodes first by community, and then within each community, sort the nodes by within community degree. This is what's seen in the second plot from the left of Figure \ref{f:LDF}. There are instances where one might want to sort the nodes differently, for example, if there is a core-periphery structure, it might be preferable to sort nodes first by community then by weight of edges connected to nodes in the core. In the third plot from the left, we show an ``estimate" of the SC term from our method, with the same ordering as in the second plot. Finally, a bootstrap replicate network of the original network is displayed in the right plot, again reordered for easier viewing. Notably, based on the different contour shapes in the bottom left and the top right sections of the third plot, it can be seen (using the plots in Figure \ref{f:different-h-funcs} as a point of reference) that the intra-left hemisphere edges have a different best fitting $H$-function than the intra-right hemisphere edges.      
\begin{figure}
 \includegraphics[width=\linewidth]{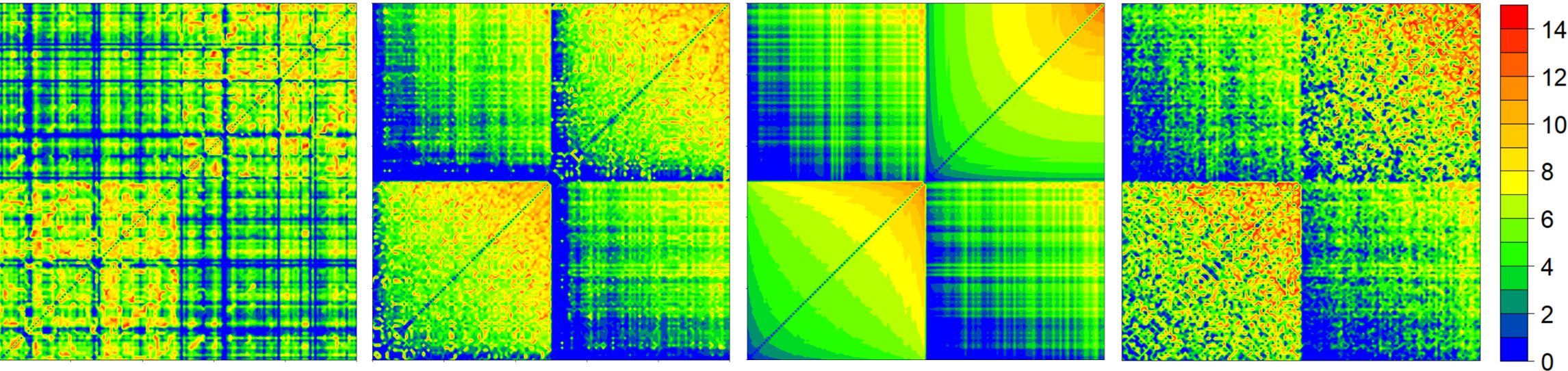}
\caption{From left to right: a structural brain network of the log values of white matter fiber counts between 148 regions. The same network reordered by within community degree. The similarly reordered ``estimate" of the structural network. A reordered bootstrap replicate network of the observed network.}
\label{f:LDF}
\end{figure}

There are other models designed to generate weighted networks. The Weighted Stochastic Block Model introduced by \cite{aicher2013adapting} includes degree correction only with regard to an edge's existence, not for modeling the weights of particular edges. The generalized exponential random graph model from \cite{desmarais_2012} is indeed a very general model, but requires a lot of advanced knowledge to specify the appropriate model for estimation if given a specific network. As noted above, \cite{Fosdick_Hoff_2015} includes a form that looks superficially like the linear models discussed in this paper, but the higher order dynamics described when incorporating multiplicative interaction effects bears little resemblance to the “non-linear” models presented here, and doesn't accomodate communities. \cite{peixoto2018nonparametric} looks for general forms of community structure, but not of the kind proposed here, as their edge weights depend only on community membership without regard for other nodal features. For other work on weighted networks, see the melding of of mutual information and common neighbors in \cite{Zhu_Xia_2016}, estimates of nodes' perceptions of one another to predict signed edge weights in \cite{Kumar_Spezzano_Subrahmanian_Faloutsos_2016}, and leveraging graph metrics to ``denoise" weighted networks in \cite{spyrou2018weighted}. Related work in the unweighted setting can be found in \cite{Bartlett_2017}, which deploys pairwise measures of node association to model binary edges,  and the use of copulas in \cite{Fan_Xu_Cao}.

As seen in Figure \ref{f:LDF}, for a given network, we can use our model to estimate a data generating process and subsequently generate synthetic data via a bootstrap-type method. This procedure can create ``new" networks which replicate the structure of the observed network even without a priori knowing the functional form of the edge weight generating process between particular communities. As in the case of using brain networks as a diagnostic aid, when networks are used as inputs to other analyses, if data collection is difficult, these synthetic network replicates may be used as supplemental data. Additionally, taking a cue from the rightmost plot of Figure \ref{f:LDF}, the random variation between bootstrap replicate networks can serve as a sensitivity test for results using the original network, allowing for greater robustness even with limited data, as in the classical bootstrap.  

The rest of this paper is structured as follows. In Section \ref{s:cosmo}, we develop theory for generating the proposed class of networks, along with details on $H$-functions in Section \ref{s:Hfunc}. In Section \ref{s:intronmici}, we discuss methods for estimating the generating processes of observed networks when the community memberships of each node are known or have been estimated. In Section \ref{s:bootstrap}, we build on the estimation procedures from Section \ref{s:intronmici} to generate new synthetic networks that are plausible stand-ins for real networks, in the vein of the bootstrap. In Section \ref{s:comm-det}, we discuss how to adapt community detection techniques to networks of this kind. In Section \ref{s:robust}, we discuss applying our methods under slight departures from the main models of interest. In Section \ref{s:net-sims}, we apply our method to real data and compare the performance to other existing models. The appendix discusses technical details and extensions.

\section{Model formulation}\label{s:cosmo}
Let $v_1,\ldots, v_n$ be the vertices (nodes) in a dense network with undirected and weighted edges, and no self-loops. Each node belongs to exactly one community, $1, \ldots , K$. Henceforth, $u$ and $v$ will refer to nodes, and $i$ and $j$ will refer to communities, e.g. $u \in i, v \in j$.

In our random graph model, a node $u$ has a ``sociability" (i.e.\ ``popularity", degree correction) parameter $\Psi_u$. These parameters are assumed to be i.i.d.\ U$(0,1)$. Let $W_{u v}$ denote the weight of the edge connecting nodes $u$ and $v$. We primarily focus on the case with continuous-valued $W_{u v}$. At the most general level, we examine random graphs of the following form: for $u, v$ such that $u \in i$, $v \in j$, suppose
\begin{equation}\label{e:NSM}
f_{ij}(W_{u v}) = h_{i j}(\Psi_{u}, \Psi_{v}) + \sigma_{ij}\epsilon_{uv},
\end{equation}
 where $f_{ij}$ is a monotonically increasing function over the range of $W_{uv}$, $h_{ij}$ is a monotonic function in its 2 arguments, $\epsilon_{uv}$ are error terms with $\mathbb{E}(\epsilon_{uv}) = 0$,  $\mathbb{E}(\epsilon_{uv}^2) = 1$ and $ \sigma_{ij} \ge 0$. Those terms with $ij$ subscripts are particular to edges where one of the nodes is in community $i$ and the other is in community $j$, while terms with subscripts $uv$ are idiosyncratic for that particular edge. In the most flexible version of the model, as in the PABM, each node may have $K$ different $\Psi$ values, each one for parametrizing edge weights connecting to nodes in a particular community. A special case of interest is the linear model
\begin{equation} \label{e:LSM}
f_{ij}(W_{u v}) = \gamma_{i j} + \alpha_{ij}h_{1ij}(\Psi_u) + \beta_{ij}h_{2ij}(\Psi_v) + \sigma_{ij}\epsilon_{uv},
\end{equation}
where $ \gamma_{i j},  \alpha_{ij},  \beta_{ij} \in \mathbb{R}$, $\mathbb{E}(h_{kij}(\Psi)) = 0$,  and $\mathbb{E}(h_{kij}(\Psi)^2) = 1$ for $k = 1,2$ and $\Psi \sim$ U$(0,1)$. We refer to (\ref{e:LSM}) as a \textit{linear sociability model (LSM)} and to (\ref{e:NSM}) where $h_{ij}$ is not linear as a \textit{nonlinear sociability model (NSM)}. Examples and discussion below will provide motivation and intuition about these models. 

Under monotonicity assumptions, note that the observed edge weight $W_{uv}$ is monotone in the sociability parameters of  nodes $u$ and $v$. As a special case, letting $\alpha_{ij} = \beta_{ij} = 0$ in the LSM, node sociability plays no role in the weight of the edge between nodes $u$ and $v$, rather the weights are generated independently from some distribution, as in a Weighted Stochastic Block Model. Similarly, letting $\sigma_{ij} = 0$ would generate a network completely determined by random node sociabilities. 

In what follows, $\Phi_{\mu,\sigma^2}$ will denote the CDF of a $\mathcal{N}(\mu, \sigma^2)$ distribution and $\Phi_{\sigma^2} := \Phi_{0,\sigma^2}$. Furthermore, though each pair of communities $i$ and $j$ are assumed to possibly be connected via a function $h_{ij}$ (along with $\gamma_{i j},  \alpha_{ij},  \beta_{ij}$, etc.), to simplify notation, we will drop the subscript in our notation, assuming that the discussion always concerns the relevant pair of communities $i, j$ based on context, where $i$ may or may not be the same as $j$.

\begin{example}(Normal LSM.)  This is (\ref{e:LSM}) with 
\begin{align}\label{e:NLSM2}
f(W_{uv}) & = \gamma + \alpha \Phi_1^{-1}(\Psi_u) + \beta \Phi_1^{-1}(\Psi_v) + \sigma \epsilon_{uv} \nonumber\\
& =  \gamma + \alpha Z_u + \beta Z_v + \sigma \epsilon_{uv},
\end{align}
where $Z_u \sim \mathcal{N}(0, 1)$. The function $f$ can be the identity (in which case $W_{uv}$ is Gaussian itself, assuming normality of $\epsilon_{uv}$) or some other transformation, such as $f(W) = \log(W)$.
\end{example}

One natural choice of $f$ in (\ref{e:NSM}) or (\ref{e:LSM}), after a common practice of transforming data to standard normal, is to consider 
\begin{equation}\label{e:normalize}
 f(w)=  \Phi_1^{-1}(G(w)),
\end{equation}
where $G$ represents the CDF of $\{W_{uv}: u \in i, v \in j\}$. We pursue this case in the following canonical example that we use for NSMs. The example relies upon \textit{H-functions}, a concept that will be discussed in greater detail in Section \ref{s:Hfunc}.

\begin{example}($H$-Normal NSM.)  This is (\ref{e:NSM}) with 
\begin{equation}\label{e:HNSM}
 \Phi_1^{-1}(G(W_{uv})) = \frac{1}{\sqrt{1+\sigma^2}}  \Phi_1^{-1}(H(\Psi_u, \Psi_v))+ \frac{\sigma}{\sqrt{1+\sigma^2}} \epsilon_{uv},
\end{equation}
where $G$ is again the CDF of $W_{uv}$, $\epsilon_{uv}$ are i.i.d.\  $\mathcal{N}(0, 1)$, and $\sigma \ge 0$. Furthermore, $H(x, y)$ is an $H$-function having the following key properties (see Section \ref{s:Hfunc} for more details): $H(\Psi_u, \Psi_v) \sim$ U$(0,1)$ for independent U$(0,1)$ random variables $\Psi_u, \Psi_v$, and $H(x, y)$ is monotone in both arguments. The first property ensures that 
\begin{equation}
\frac{1}{\sqrt{1+\sigma^2}} \Phi_1^{-1}(H(\Psi_u, \Psi_v)) + \frac{\sigma}{\sqrt{1+\sigma^2}} \epsilon_{uv}\ =: Z_{uv}
\end{equation}
are $\mathcal{N}(0, 1)$ variables, and hence by inverting (\ref{e:HNSM}), the variables 
\begin{equation}
W_{uv} = G^{-1}(\Phi_1(Z_{uv}))
\end{equation}
indeed have $G$ as their CDF. 
\end{example}

Note that the $H$-Normal NSM also has the Normal LSM  as a special case. Using the $H$-function 
\begin{equation}\label{e:H-normal-rho}
H(x, y) = \Phi_{1+\rho^2}(\Phi_1^{-1}(x) + \Phi_{\rho^2}^{-1}(y)),
\end{equation}
one observes that $H(\Psi_u, \Psi_v) \sim$ U$(0,1)$ since $\Phi_1^{-1}(\Psi_u) \sim \mathcal{N}(0, 1),   \Phi_{\rho^2}^{-1}(\Psi_v) \sim \mathcal{N}(0, \rho^2)$, and hence $\Phi_1^{-1}(\Psi_u) + \Phi_{\rho^2}^{-1}(\Psi_v)  \sim \mathcal{N}(0, 1+\rho^2)$. With the choice (\ref{e:H-normal-rho}) plugged into (\ref{e:HNSM}), the latter model becomes 
$$\Phi_1^{-1}(G(W_{uv})) = \frac{1}{\sqrt{1+\sigma^2}}  \Phi_1^{-1}(H(\Psi_u, \Psi_v))+ \frac{\sigma }{\sqrt{1+\sigma^2}} \epsilon_{uv} $$
$$= \frac{1}{\sqrt{1+\sigma^2}} \left( \Phi_1^{-1}(\Phi_{1+\rho^2}(\Phi_1^{-1}(\Psi_u) + \Phi_{\rho^2}^{-1}(\Psi_v))) \right)+ \frac{\sigma }{\sqrt{1+\sigma^2}} \epsilon_{uv} $$
\begin{equation}\label{e:H3-with-noise}
=\frac{1}{\sqrt{1+\sigma^2} \sqrt{1+\rho^2}} \Phi_1^{-1}(\Psi_u) + \frac{\rho}{\sqrt{1+\sigma^2 }\sqrt{1+\rho^2}} \Phi_1^{-1}(\Psi_v) + \frac{\sigma}{\sqrt{1+\sigma^2}} \epsilon_{uv},
\end{equation}
by using the identities $\Phi_a^{-1}(c) = \sqrt{a}\Phi_1^{-1}(c)$ and $\Phi_b(c) = \frac{\Phi_1(c)}{\sqrt{b}}$. Note that (\ref{e:H3-with-noise}) is a normalized version of the Normal LSM (\ref{e:NLSM}) where $\gamma$ is set to 0. Examples of $H$-functions which do not correspond to Normal LSM will be given in Section \ref{s:Hfunc}. NSMs are a large class, but some other potentially interesting examples can be constructed in a similar manner to the $H$-Normal NSM above, as is detailed in the technical appendix. 

While $H$-Normal NSMs indeed take advantage of many features of the normal distribution, they are actually not very restrictive. Instead of representing the CDF of a linear combination of normal random variables as in (\ref{e:H-normal-rho}), the $H$-function in (\ref{e:HNSM}) can represent the CDF of some other weighted combination of random variables, in which case the underlying ``shape" of the connections between communities $i$ and $j$ will look very different, as can be seen in Figure \ref{f:different-h-funcs}. This paper will focus on $H$-Normal NSMs because all $H$-Normal NSMs incorporate normally distributed errors. 

Finally, we introduce a bit more terminology. In the LSM (\ref{e:LSM}), we distinguish the following cases with specific terms:
\begin{itemize}
  \item $\alpha > 0$, $\beta > 0$: positive association,
  \item $\alpha < 0$, $\beta < 0$: negative association,
   \item $\alpha \cdot \beta < 0$: Simpson association,
  \item $\alpha \ne 0$, $\beta = 0$: projection onto 1st coordinate,
   \item $\alpha = 0$, $\beta \ne 0$: projection onto 2nd coordinate.
\end{itemize}

\section{$H$-functions}\label{s:Hfunc}
We begin by introducing some terminology. 
\begin{definition} \label{d:pos-ass} (Positive association.) 
A function $H:(0, 1) \times (0,1) \rightarrow (0,1)$ is an $H$-function with \textit{positive association} if:
\begin{enumerate}
  \item $H$ is non-decreasing in both arguments;
  \item  $\iint_{H(x, y)\le z} dx dy =z$, for all $z \in (0,1)$.
\end{enumerate} 
\end{definition}
The term ``positive association" refers to the fact that, when considered across communities, such models would tend to produce larger weights for nodes in the two communities with simultaneously larger sociabilities. The monotonicity condition 1 captures the idea of node sociability as discussed above. Condition 2 is equivalent to requiring that $H(\Psi_u, \Psi_ v) = \Psi_{u v}$ is a U$(0,1)$ random variable. In contrast with copulas, the output of a positively associated $H$-function is not bounded above by the minimum of the inputs. 
 
\begin{definition}\label{d:neg-ass} (Negative association; Simpson association.) A function $H:(0, 1) \times (0,1) \rightarrow (0,1)$ is an $H$-function with \textit{negative association} if $H(1-x, 1-y)$ is an $H$-function with positive association. A function $H:(0, 1) \times (0,1) \rightarrow (0,1)$ is an $H$-function with \textit{Simpson association} if $H(x, 1-y)$ or $H(1-x, y)$ is an $H$-function with positive association. 
\end{definition}

If $\Psi_u$ is a uniform random variable, then $1-\Psi_u$ is also a uniform random variable, so negative association also ensures that $H(\Psi_u, \Psi_ v) = \Psi_{u v}$ is a uniform random variable.  A similar observation can be made for Simpson association. The term ``negative association" arises because the monotonicity of $H$ results in the fact that, when looking across the communities, nodes with greater node sociabilities actually tend to have smaller edge weights. Simpson associations are so named because they indicate a localized area where certain broader trends of the network may be inverted. This error at the local level when extrapolating from global phenomena is reminiscent of Simpson's paradox.    

A property shared by all $H$-functions is that $H(\Psi_{u}, \Psi_{v})$ is a U$(0, 1)$ random variable for such independent random variables $\Psi_u, \Psi_v$. There are many ways to achieve this, but one quite general construction which we found to be flexible and interesting is as follows. Note that a random variable $F^{-1}(\Psi)$ has the CDF $F$ for a U$(0,1)$ random variable $\Psi$. Take now two CDFs $F_1, F_2$  and let $F_{1,2} = F_1 * F_2$ be their convolution CDF. Then $F_1^{-1}(\Psi_u) + F_2^{-1}(\Psi_v)$ has the same distribution as  $F_{1,2}^{-1}(\Psi_{u v})$. This suggests setting
\begin{equation}\label{e:HCDF}
H(x, y) = F_{1,2}(F_1^{-1}(x) +F_2^{-1}(y)).
\end{equation}
By construction, this function satisfies the condition 2 of Definition \ref{d:pos-ass}, but one can easily check that condition 1 holds as well. The function (\ref{e:H-normal-rho}) is an example of $H$-function in the form of (\ref{e:HCDF}) with an explicit convolution $F_{1,2}$. Besides the normal distributions, choosing $F_1$ and $F_2$ to be exponential, Cauchy, or uniform would also give an explicit form for $F_{1,2}$, although (\ref{e:HCDF}) is far more general than these simple cases imply. 

\begin{figure}
 \includegraphics[width=\linewidth]{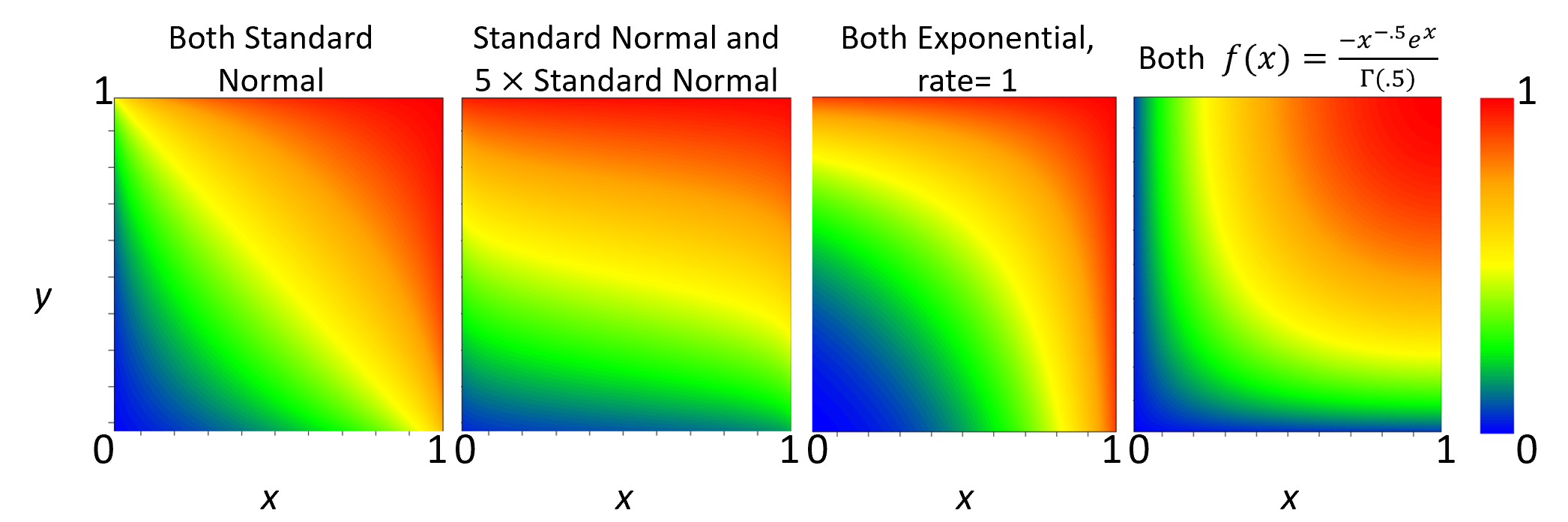}
\caption{Plots of examples of $H$-functions.}
\label{f:different-h-funcs}
\end{figure}

Figure \ref{f:different-h-funcs} illustrates some of the different kinds of contours that can be created using $H$-functions of the form in (\ref{e:HCDF}). The resulting weighted bipartite subnetworks between 2 different communities generated using these $H$-functions in $H$-Normal NSM (\ref{e:HNSM}) would inherit similar connectivity patterns, albeit with normally distributed ``errors" included. In this case, the x and y axes represent the values of $x$ and $y$, respectively, each running from .01 to .99 by increments of .01, and the colors represent the output of the $H$-function. From left to right, the first plot shows the values of $H(x, y) = \Phi_{1+\rho^2}(\Phi_1^{-1}(x) + \Phi_{\rho^2}^{-1}(y))$  where $\rho =1$. The second plot corresponds to this function with $\rho = 5$. The third plot depicts (\ref{e:HCDF}) where $F_1$ and $F_2$ are both exponential distributions with a rate parameter 1, and $F_{1,2}$ is a gamma distribution with shape parameter of 2 and rate parameter of 1. Finally, the rightmost plot is from the $H$-function (\ref{e:HCDF}) where  $F_1$ and $F_2$ have density

$$f(x) =
  \begin{cases} \frac{-x^{-.5}e^x}{\Gamma(.5)}, & x< 0, \\
  0, & \text{otherwise}. 
   \end{cases}
$$
\\
In this case, $F_{1,2}$ can be checked to be given by
$$F_{1,2}(z) =
  \begin{cases}
 e^{z}, & z< 0, \\
  1, & z\ge 0 . 
   \end{cases}
$$

These different images show that $H$-functions (\ref{e:HCDF}) make a rather flexible class. Other $H$-functions include maps to the first or second coordinates, which would give perfectly vertical or horizontal contours.

\section{Estimation with known communities}\label{s:intronmici}

In this section, we discuss estimation of the different models discussed in Section \ref{s:cosmo}, while assuming that the true community labels $\{ i \}$ of the nodes in the network are known. As far as estimation goes, we do not impose that each node's estimated $\widehat{\Psi}_u$ value is constant globally, but rather only constant over each community. It may be desirable in future work to align these $\widehat{\Psi}_u$ estimates over the whole network, but the presented estimation processes are more flexible. 
 Even using this assumption, the model parameters to be estimated depend on the specific model in question. Given a particular set of community labels, we can treat each subnetwork of the larger network -- where we analyze the connectivity patterns between 2 different communities $i$ and $j$ -- as a bipartite graph, and any subnetwork where we look at the connectivity within a single community as a smaller network.

\subsection{Estimation for Normal LSM}\label{s:nmicim1}
 We assume henceforth that $i$ and $j$ are fixed and work on one smaller subnetwork. It is assumed that the transformation $f$ in (\ref{e:NLSM2}) has already been performed, so without loss of generality,  $f(W_{uv}) =W_{uv}$. We also assume for simplicity that $\gamma =0$. Then, the model (\ref{e:NLSM2}) can be expressed after exponentiation as
\begin{equation}\label{e:exp_edge}
e^{W_{uv}} = e^{\alpha Z_u}e^{\beta Z_v}e^{\sigma \epsilon_{uv}} = e^{\alpha Z_u}e^{\beta Z_v}e^{\frac{\sigma^2} {2}}\left(\frac{e^{\sigma \epsilon_{uv}}}{\E e^{\sigma \epsilon_{uv}}}\right),
\end{equation}
where the last term within the parentheses has an expected value of 1. The structure of (\ref{e:exp_edge}) enables the use of rank-one Nonnegative Matrix Factorization (NMF) to estimate the parameters of interest. NMF approximates the matrix represented by $e^{W_{uv}}$ as the decomposition $ab^\prime$ where $a = (a_u)$ and $b = (b_v)$ are column vectors with positive entries. Taking the log of this approximation yields the approximate identity
\begin{equation}\label{e:NMF}
\log(a_u) + \log(b_v) \approx \alpha Z_u +\beta Z_v+ \frac{\sigma^2}{2},
\end{equation}
which suggests the following estimators for the model parameters of interest:
\begin{itemize}
  \item[] $\widehat{\alpha}$ = SD$(\log(a))$,  \, $\widehat{\beta}$ = SD$(\log(b))$,

    \item[] $\widehat{Z}_u =  \frac{\log(a_u) - \overline{\log(a)}}{\widehat{\alpha}}$, \,  $\widehat{\Psi}_u = \Phi_1(\widehat{Z}_u)$, \,   $\widehat{Z}_v =  \frac{\log(b_v) - \overline{\log(b)}}{\widehat{\beta}}$, \, $\widehat{\Psi}_v = \Phi_1(\widehat{Z}_v)$,
\end{itemize} 
where SD stands for the standard deviation, and $\overline{\log(x)}$ indicates the sample mean of $\log(x)$. In light of (\ref{e:NLSM2}), we also set $\widehat{\sigma}$ = SD$(W_{uv} - \widehat{\alpha}  \widehat{Z}_u- \widehat{\beta} \widehat{Z}_v)$.  An adjustment for subnetworks where all nodes are in the same community is given in Appendix \ref{sa:repeats}.

A concentration inequality for a Normal LSM bounding the difference between the best possible estimates of the network's SC to the true generating SC process without any ``error" included is given in Appendix \ref{sa:LSM-proof}.

\subsection{Extension to LSM}\label{s:nmicim2}

The difference between (\ref{e:NLSM2}) and (\ref{e:LSM}), ignoring extra subscripts, is that rather than having standard normal random variables ascribed to each node, (\ref{e:LSM}) includes random variables $h_1(\Psi_u)$ and $h_2(\Psi_v)$ with possibly different forms, albeit with identical first 2 moments. The procedure described in Section \ref{s:nmicim1} will still apply with one exception. The relation (\ref{e:NMF}) cannot directly estimate $\widehat{Z}_u$ or $\widehat{Z}_v$, but rather $\widehat{h_1(\Psi_u)}$ and $\widehat{h_2(\Psi_v)}$. After getting these estimates, a distribution can be fit to the data points while assuming that $\Psi_u$ and $\Psi_v$ are truly distributed uniformly over the unit interval. The best fitting distribution can then be inverted to estimate $\widehat{\Psi}_u$ and $\widehat{\Psi}_v$. Finally, the parameter $\sigma$ is estimated to be the standard deviation of 
$W_{uv} - \widehat{\alpha}  \widehat{h_1(\Psi_u)}- \widehat{\beta} \widehat{h_2(\Psi_v)}$.

\subsection{Estimation for $H$-Normal NSM}\label{s:nmcomp}   

Though using NMF is appropriate when the function $h$ in (\ref{e:NSM}) is linear, it is unsuitable for nonlinear functions, which require an alternative methodology.
For an estimator of $G_{ij}$, by the construction of the graph in the $H$-Normal NSM (\ref{e:HNSM}), one could naturally set the empirical CDF of the weights $\{W_{uv}: u \in i, v \in j\}$. We make a small modification to this and instead set
\begin{equation}\label{e:edge-ECDF}
\widehat{G}(w) = \frac{1}{n+1}\sum_{W_{uv}:\, u \in i, \, v \in j} \mathbbm{1}_{\{W_{uv}\le w\}}\, ,
\end{equation}
where $n$ is the number of edges $W_{uv}: u \in i, v \in j$. That is, we divide by  $n+1$ in (\ref{e:edge-ECDF}) instead of $n$. There are two reasons for this. First, note that the model in (\ref{e:HNSM}) implies a distorted but generally linear relationship between $\Phi_1^{-1}(G(W_{uv}))$ and $\Phi_1^{-1}(H(\Psi_u, \Psi_v))$. We shall use this relation to estimate the function $H$, and at the empirical level, shall consider $\widehat{G}(W_{uv})$ in place of $G(W_{uv})$. Dividing by $n+1$ ensures that $\widehat{G}(W_{uv}) < 1$, so $\Phi_1^{-1}(\widehat{G}(W_{uv}))$ is finite. Second, modulo any dependence issues, if one thinks of (an appropriate scaling of) $H(\Psi_u, \Psi_v)$ as representing the order statistics of $n$ uniform random variables on (0,1), recall that the $k$th order statistic follows a Beta$(k, n+1-k)$ distribution, which has a mean of $\frac{k}{n+1}$. At the mean level, it is then natural to place $\widehat{G}(W_{uv})$ at multiples of $\frac{1}{n+1}$, not $\frac{1}{n}$.
In fact, we use one other modification to the definition (\ref{e:edge-ECDF}) when the values $W_{uv}$ repeat, which can be found in Appendix \ref{sa:repeats}.

For node sociabilities $\Psi_u$, we define them locally based on two communities $i$ and $j$ (and possibly $i=j$ so there is only one community). For node $u \in i$, consider 
 \begin{equation}\label{e:local-degree}
D_{j}(u) = \sum_{v:\, v \in j}\Phi_1^{-1}(\widehat{G}(W_{uv})).
\end{equation}
 We think of $D_{j}(u)$ as a ``local sociability statistic" of $u$, since it looks at how $u$ connects to one community $j$, rather than the whole network. By the construction of the NSM model (\ref{e:HNSM}) and the properties of positively associated $H$-functions, if $H$ is positively associated, one expects the ordering of the local sociability statistics $D_{j}(u)$'s of those nodes in community $i$ to match the ordering of the sociabilities $\Psi_u$. This suggests setting 
 \begin{equation}\label{e:psi-est}
\widehat{\Psi}_u^{(j)} = \frac{1}{n_i+1}\sum_{u': u' \in i} \mathbbm{1}_{\{D_j(u')\le D_j(u)\}},
\end{equation}
where $n_i$ is the number of $u' : u' \in i$. That is, defining $\widehat{\Psi}_u^{(j)}$ as the rescaled ordering of the ``local sociability statistic" of $u$ in its community $i$ with respect to community $j$. As in (\ref{e:edge-ECDF}), note the division by $n_i +1$ in (\ref{e:psi-est}), placing the  $\widehat{\Psi}_u^{(j)}$ values at the expected values of the order statistics of $n_i$ draws from a U$(0,1)$ distribution. When the association of $H$ is negative, we expect the ordering of these local sociability statistics $D_{j}(u)$ to have a strong negative correlation with the true $\Psi_u$ values. In other words, if the true $H$ has negative association, we expect the ordering of the local sociability statistics $D_{j}(u)$'s of those nodes in community $i$ to match the ordering of $1 - \Psi_u$. The estimation of $H$ described next will therefore adapt automatically to any form of the association of $H$.

We view the estimation of $H$ as choosing the best candidate from a set $\mathcal{H}$ of $H$-functions. This set can be parametric (e.g. parametrized by $\rho^2$ in (\ref{e:H-normal-rho})) or consist of several $H$-functions. More precisely, we set:
\begin{equation}\label{e:optimization}
\begin{aligned}
\widehat{H}, \widehat{\sigma} =  & \argmin {H\in \mathcal{H}, \sigma \ge 0} \sum\limits_{u \in i, v \in j} \left(\Phi_1^{-1}(\widehat{G}(W_{uv})) - \frac{1}{\sqrt{1+\sigma^2}}\Phi_1^{-1}( H(\widehat{\Psi}_u^{(j)}, \widehat{\Psi}_v^{(i)}))\right)^2.
\end{aligned}
\end{equation}
The optimization in  (\ref{e:optimization}) is carried out numerically over different functional forms of $H$, and the associated minimizing choice of $\sigma$ is taken as the estimated $\widehat{\sigma}$.

\subsection{Estimated network sociability}\label{s:examps}

There may be a desire to examine the SC of the network implied by the parameter estimates. In this case, the ``estimated edge values" are given by
\begin{equation}\label{e:estimate}
\widehat{W}_{uv} =   \argmin {w} \left|\Phi_1^{-1}(\widehat{G}(w)) - \frac{1}{\sqrt{1+\widehat{\sigma}^2}}\Phi_1^{-1}(\widehat{H}(\widehat{\Psi}_u^{(j)}, \widehat{\Psi}_v^{(i)}))\right|.
\end{equation}
These estimates seek to smooth out the effects of any ``errors" observed over particular edges in the original network. As $\widehat{\sigma}$ grows, $\frac{1}{\sqrt{1+\widehat{\sigma}^2}}\Phi_1^{-1}(\widehat{H}(\widehat{\Psi}_u^{(j)}, \widehat{\Psi}_v^{(i)}))$ shrinks to 0, so the estimate $\widehat{W}_{uv}$ tends toward the median edge weight in the subnetwork. If a subnetwork has a large estimated $\widehat{\sigma}$ value, the range of the estimated subnetwork will be much smaller than the range of the observed subnetwork. By contrast, bootstrap replicates of the kind described in Section \ref{s:bootstrap} below are expected to have the same variance structure as the original network.

\subsection{Spurious patterns}\label{s:spurious}
Note that the $H$-Normal NSM (\ref{e:HNSM}) allows for the independent edges $\epsilon_{uv}$ only in the limit $\sigma \rightarrow \infty$. In practice, even when only independent edges $\epsilon_{uv}$ are present, a finite value of $\sigma$ will be estimated, and a spurious sociability pattern will be ``found." (This is discussed in connection with Figure \ref{f:planted} below). This scenario could be flagged by examining suitable MSEs.

Using results of (\ref{e:optimization}), the MSE in ``normal" space is defined as
\begin{equation}\label{e:MSE}
\frac{1}{N} \sum\limits_{u \in i, v \in j} (\Phi_1^{-1}(\widehat{G}(W_{uv})) - \frac{1}{\sqrt{1+\widehat{\sigma}^2}}\Phi_1^{-1}( \widehat{H}(\widehat{\Psi}_u^{(j)}, \widehat{\Psi}_v^{(i)})))^2,
\end{equation}
where $N$ represents the number of edges connecting nodes in $i$ to nodes in $j$. If $\widehat{\sigma} \rightarrow \infty$, then by construction,  MSE $\rightarrow$ 1. With no upper bound on $\sigma$, in practice we expect the MSE for independent edges to be slightly smaller than 1. We can compare the observed MSE of a subnetwork to the MSE values we get when edge weights really are generated as independent $\epsilon_{uv}$.
 
Where overfitting is suspected in a particular subnetwork, we can draw completely random $\mathcal{N}(0,1)$ edge weights and create a fictional subnetwork of the same size as the observed subnetwork. From there, we repeat the estimation process to calculate the MSE from this synthetic subnetwork, with the additional restriction that, using the terminology of (\ref{e:HCDF}), $F_1$ and $F_2$ of the estimated $H$-function for the fictitious subnetwork must be of the same distributional family as the $F_1$ and $F_2$ in the $\widehat{H}$ estimated for the real data. 
We can generate many fictitious subnetworks, and if the MSE obtained from the true subnetwork is smaller than some large proportion of the fictional MSE values, the estimates from (\ref{e:estimate}) should be retained. Otherwise, we replace all estimated edge weights in the subnetwork with the median edge weight in the subnetwork.

\section{Bootstrap}\label{s:bootstrap}
In certain cases, such as brain scans, it can be difficult to obtain multiple measurements of the same network where the underlying structure broadly remains the same, but there may be some variation at the level of particular edges. For this reason, we want a procedure which can generate new networks that can mimic the structure of real networks, much the same way the classical bootstrap can be used to generate new samples from a single sample.  In Section \ref{s:intronmici}, we estimated a network using $H$-functions. In this section, we extend our construction to generate new network samples that still allow for both flexible connectivity patterns between communities as well as node specific degree correction effects. 

Assuming the community assignments are correct, pairwise functions $\widehat{H}$ and parameters $\widehat{\sigma}$ are estimated based on the estimated sociabilities $\widehat{\Psi}_u^{(j)}$. To get a bootstrapped edge weight, we can draw a new $\widetilde{\epsilon}_{uv} \sim \mathcal{N}(0, 1)$  in (\ref{e:HNSM}) and set  


\begin{equation}\label{e:bootstrap}
\widetilde{W}_{uv} =   \argmin {w} \left |\Phi_1^{-1}(\widehat{G}(w)) - \left(\frac{1}{\sqrt{1+\widehat{\sigma}^2}}  \Phi_1^{-1}(\widehat{H}(\widehat{\Psi}_u^{(j)}, \widehat{\Psi}_v^{(i)}))+ \frac{\widehat{\sigma}}{\sqrt{1+\widehat{\sigma}^2}} \widetilde{\epsilon}_{uv}\right)\right|.
\end{equation}

When the true $\sigma$ is small, $\widehat{\sigma}$ can be estimated to be 0. However, for the purposes of the bootstrap, it is useful to include randomness; otherwise each bootstrap replicate network will be identical. In the case where $\widehat{\sigma} < c$ (a small positive value), we can replace $\widehat{\sigma}$ in (\ref{e:bootstrap}) with the MSE given by (\ref{e:MSE}). 

If, after performing the procedure described in Section \ref{s:spurious}, we believe there is no relationship between edge weights and their incident nodes, for each bootstrap replicate, we instead draw every edge at random with replacement from the relevant edge set. In that case $$\widetilde{W}_{uv} =  \argmin {w} \ \left |\Phi^{-1}(\widehat{G}(w)) - \widetilde{\epsilon}_{uv}\right|.$$ 

\section{Community detection}\label{s:comm-det}
Estimates above depend on assigning each node into its community. In this work, community $i$ is defined as a subset of nodes which all share a common $f_{ij}$, $h_{ij}$ and $\sigma_{ij}$ for each particular corresponding community $j$, as given in (\ref{e:NSM}). As nodes in the same community share functions to generate edge weights, patterns in edge weights can be used to cluster nodes into communities. As estimating $\widehat{G}$ requires defining the estimated node set in each community, perhaps surprisingly, the clustering techniques discussed in this section do not depend on $\widehat{G}$, but instead rely upon a measure of cluster goodness.

\subsection{Measure accounting for sociability}\label{s:icim1}
Letting $\sigma = 0$ and fixing any two communities $i$ and $j$, for any model discussed in Section \ref{s:cosmo}, with $u_1, u_2 \in i$, $v_1, v_2 \in j$, if $W_{u_1 v_1} > W_{u_1 v_2}$, then so too is $W_{u_2 v_1} > W_{u_2 v_2}$. Nodes in the same community $i$ share an ``order of preferences" over nodes in another particular community $j$, as reflected by persistently greater edge weights. Even allowing for positive $\sigma$, a good clustering for these models should reflect a relatively consistent order of preferences. Since this order of preferences over nodes in $j$ is expected for all nodes in community, the local degree of each node $v$ in community $j$ with respect to community $i$, $$d_i(v)  = \sum_{u':\, u' \in i}W_{u'v},$$ should also display this ordering. For example, in any plot in Figure \ref{f:different-h-funcs}, comparing any set of columns, the rightmost column (representing the node with the larger $\Psi$ value) never has a smaller edge weight value than the corresponding location in the left column. That is, if $W_{u_1v_1} > W_{u_1v_2}$, then $d_i(v_1) > d_i(v_2)$.  For each node $u \in i$ and each community $j$, we can then define a node-community correlation as  
$$C_{ij}(u) = corr\{(d_i(v), W_{uv}): v \in j, u \ne v\}.$$ 

For a good clustering with estimated communities $\{\, \widehat{i} \,\}$, fixing any node $u \in \widehat{i}$ and looking at all nodes $v \in \widehat{j}$, the edge weights $W_{uv}$ should be correlated with  $d_{\widehat{i}}(v)$. 
A good estimated clustering should result in large, positive $C_{\widehat{i}\,\widehat{j}}\,$ values everywhere. Even so, when $\sigma > 0$, even nodes in the same community $i$ may exhibit minor variation in preferences over nodes in community $j$. Additionally, we could ensure perfect correlation between a node's edge weights and its community's preferences if we made that node its own community, but that would be overly prescriptive. While individual $C_{\widehat{i}\,\widehat{j}}$ values may be useful for diagnosing localized issues, in a larger network, it is preferable to aggregate these values to get a system-wide overview of clustering success. For a particular assignment of communities, we define our measure as:  
\begin{equation}\label{e:measure}
L(\{\,\widehat{i}\,\}, \widehat{K})= \sum_{\widehat{i}=1}^{\widehat{K}}\sum_{\widehat{j}=1}^{\widehat{K}} \overline{C_{\widehat{i}\,\widehat{j}}(u)}\times\left(1-\sqrt{SD(C_{\widehat{i}\,\widehat{j}}(u))}\right) \times ((n_{\widehat{i}}-2)(n_{\widehat{j}}-2))_+ \times (1+ \mathbbm{1}_{\{\widehat{i}=\widehat{j}\}}),
\end{equation}
where $\widehat{K}$ is the total number of estimated communities in the network, and $n_{\widehat{i}}$ is the number of nodes in community $\widehat{i}$. The average and standard deviation of $C_{\widehat{i}\,\widehat{j}}(u)$ in (\ref{e:measure}) are over $u \in \widehat{i}$. 

Nodes placed in the same community should exhibit a shared ordering of preferences over nodes in any other community, represented by the average of the $C_{ij}$ values. In addition to rewarding clusterings that show consistent ordering of preferences, the measure $L$ also prefers clusterings with less variation in $C_{ij}$ values for a fixed $i$ and $j$, which can reflect a shared $\sigma$ value. By multiplying by $((n_i-2)(n_j-2))_+$, there are increasing returns to scale in the size of communities. Without increasing returns to scale, nodes could be clustered into many dyads or triads, all producing consistently large absolute $C_{ij}$ values, but this clustering would lead to overfitting. Under $L$, communities of size one or two are worthless. Finally, within community subnetwork performance is counted twice to balance the influence of all subnetworks. Consider a network with 2 communities of 52 nodes each. Without this doubling, the two within community subnetworks would each have a maximum possible contribution of 2500 to the measure, while the between community subnetwork would have a maximum possible contribution of 5000.  

The measure $L$ tries to find the appropriate balance between size and homogeneity of the estimated communities. Increasing returns to scale are crucial because they induce larger communities, even if they contain some nodes with minor deviations from the community's collective ordering of preferences. However, if multiple nodes have preferences at odds with the rest of their assigned community, it would become beneficial to separate this set of crosscutting nodes into their own splinter community to improve the totality of the measure. Of course, just as modularity may not be ideal for community detection in every network model, in cases where node sociabilities do not matter (akin to a standard SBM), this measure will not be effective at recovering the true communities.  

It's worth noting here that absence of an ordering of preferences can also be a valid shared ordering of preferences. In the simplest case, all weights between communities $i$ and $j$  can be identical, or they can all be generated as i.i.d. random variables. This still may be useful for clustering. For example, if communities $i$ and $j$ have identical functions to generate both within community and between community edges, it may be inappropriate to call them two different communities. However if the edges between $i$ and a third community $k$ are generated completely at random, but the generating process of edges between $j$ and $k$ has some kind of association, that should distinguish nodes in community $i$ from nodes in community $j$.

In Appendices \ref{s:greedy-alg} and \ref{s:spectral-alg}, we present 2 community detection algorithms which try to maximize the measure $L$. One is stochastic, while the other is bottom-up and deterministic. Experimentally, there have been occasions where each algorithm outperforms the other. Unless otherwise noted, community estimates presented in figures in this paper are the $L$ maximizing clustering given by one of these algorithms. 




\section{Robustness of estimation procedure}\label{s:robust}
The estimation pipeline described above is tailor made for the dense weighted networks described in Section \ref{s:cosmo}. However, the procedure still appears to succeed for related networks which are not explicitly NSMs or LSMs. 

\subsection{Sparser networks}\label{s:sparse}
While the discussion so far has centered on dense networks, in this section, we propose an extension where many edges may be missing, and there are two layers to the generative model. In this instance, it is necessary to distinguish between the adjacency network, which is the set of present edges, and the set of weights of those edges. In principle, the set of communities in the adjacency network could be different than the set of communities in the weights, but we only consider the case where they are the same. However, we do allow potentially different sets of sociability parameters. 
To generate a network with missing edges, we can use existing models such as the SBM, DCBM, or PABM to generate the adjacency network. To generate the edge weights, first we generate a dense weighted network as in this paper, then take the Hadamard product of the adjacency network matrix and the dense weighted network.
     
Moving from generation to estimation, we can extract the adjacency network of an observed network by replacing any non-zero weights with 1, then use established community detection and estimation methods to cluster the nodes and estimate the probability of an edge's presence. With these estimated community memberships, we can iteratively use a modified version of the estimation method described in Section \ref{s:intronmici} to estimate the generative process for the edge weights, as shown in Algorithm \ref{alg:missing}. First, we estimate the model on the observed network while ignoring missing (zero valued) edges. Then, we replace any missing edges in the original network with the estimates from our model. We then re-estimate the model based on this updated network, and continue to update those edges which were missing in the original network. This iteration is useful because a node with larger edge weights could, by chance, have many missing edges, which would deflate its estimated $\widehat{\Psi}$ value. With multiple iterations, those missing values should be replaced by better and better estimates, which should hopefully mitigate the impact of these missing values on our estimate of the edge weight generating process. This progression has been observed in simulations, and the performance of Algorithm \ref{alg:missing} is discussed in Section \ref{s:sparse_nets}.   

\begin{algorithm}\label{alg:missing}
\SetAlgoLined
\KwResult{$\hat{E}_k$, $\hat{H}$, $\hat{\sigma}$}
\KwInput{$W$, $\epsilon$}
\nl $k$ = 0; $W_0$ = $W$; $\Delta = \infty$; $\hat{E}_0$ = zero matrix of appropriate size\; 

\nl \While {$\Delta > \epsilon$}{
\nl Estimate $\hat{H}$ and $\hat{\sigma}$ for $W_k$ using the methods described in Section \ref{s:nmcomp}. When k = 0, ignore any zero valued (missing) edges in all numbered equations\;
\nl k$++$\;
\nl Using the estimated $\hat{H}$ and $\hat{\sigma}$, generate all edges in $\hat{E}_k$ via (\ref{e:estimate})\;
\nl $\Delta =||\hat{E}_k - \hat{E}_{k-1}||_F^2$ \;
\nl For any zero valued (missing) edge in $W_0$, substitute in the corresponding edge from $\hat{E}_k$ to calculate $W_k$\;
}
\caption{Estimating a subnetwork with missing edges}
\end{algorithm}

After this estimation process completes, we can move back from estimation to generation. The upshot of this entire process is that given one weighted network with missing edges, we can estimate both the adjacency network generating process and the edge weight generating process. This information can serve to generate new weighted networks with missing edges, using (\ref{e:bootstrap}) to get a synthetic edge weight network, and using the estimated SBM-type parameters based on the observed adjacency network to generate a synthetic adjacency network. Finally, take the Hadamard product of these two matrices to generate a synthetic network where the distributions for each edge weight match the estimated distribution in the original network, including treating 0 as a missing edge. 
\subsection{Noisy edge weights}\label{s:noiseedge}
Our estimation method appears to work even when a dense network is not generated as an $H$-Normal NSM. Rearranging (\ref{e:HNSM}),
 $$W_{uv} = G^{-1}\left(\Phi_1\left(\frac{1}{\sqrt{1+\sigma^2}}  \Phi_1^{-1}(H(\Psi_u, \Psi_v))+ \frac{\sigma}{\sqrt{1+\sigma^2}} \epsilon_{uv}\right)\right).$$
If $G$ is a distribution with a maximum, that means no edge weight can exceed that maximum. However, if we include another term $\zeta_{uv} \sim \mathcal{N}(0, \sigma_{\zeta}^2)$ into
 $$W_{uv} = G^{-1}\left(\Phi_1\left(\frac{1}{\sqrt{1+\sigma^2}}  \Phi_1^{-1}(H(\Psi_u, \Psi_v))+ \frac{\sigma}{\sqrt{1+\sigma^2}} \epsilon_{uv}\right)\right) + \zeta_{uv},$$
then, if $\sigma_{\zeta}^2 >0$, each edge weight is distributed over $\mathbb{R}$. As can be seen in Appendix \ref{s:noisyedgesim}, the estimation procedure still captures the underlying network dynamics in this case, although the estimates degrade as $\sigma_{\zeta}^2$ increases. 

\section{Simulations}\label{s:net-sims}
For each simulation, we include a description and a figure showing the original network and the estimated underlying network, where the clustering choice is the estimated measure $L$ maximizing clustering using the algorithms presented in Appendix \ref{a:comm-det2}. 
   
\subsection{Varying $\sigma$} 
Figure \ref{f:no-noise-NSM} depicts variations on and estimates of an underlying network with 4 communities of 37 nodes each. In plot A, the network is displayed with $\sigma = 0$.  Within community edges are drawn from a uniform distribution with a maximum of 150, while between community edges are drawn from a uniform distribution with a maximum of 100. Node sociability parameters for both communities range from .05 to .95 in increments of .025. The $H$-function is the same as that used for the rightmost plot in Figure \ref{f:different-h-funcs}, though since the between community edges have negative association, the inputs to that $H$-function are $1 - \Psi_u$ and $1 - \Psi_v$. This network is ordered so one can visually discern communities and connectivity patterns. Even so, the first step is to estimate community structure, as node ordering does not impact the community detection algorithm. The resulting estimated network is shown below the original and looks very similar to the original network.

\begin{figure}[h!]
\centering
 \includegraphics[width=\linewidth]{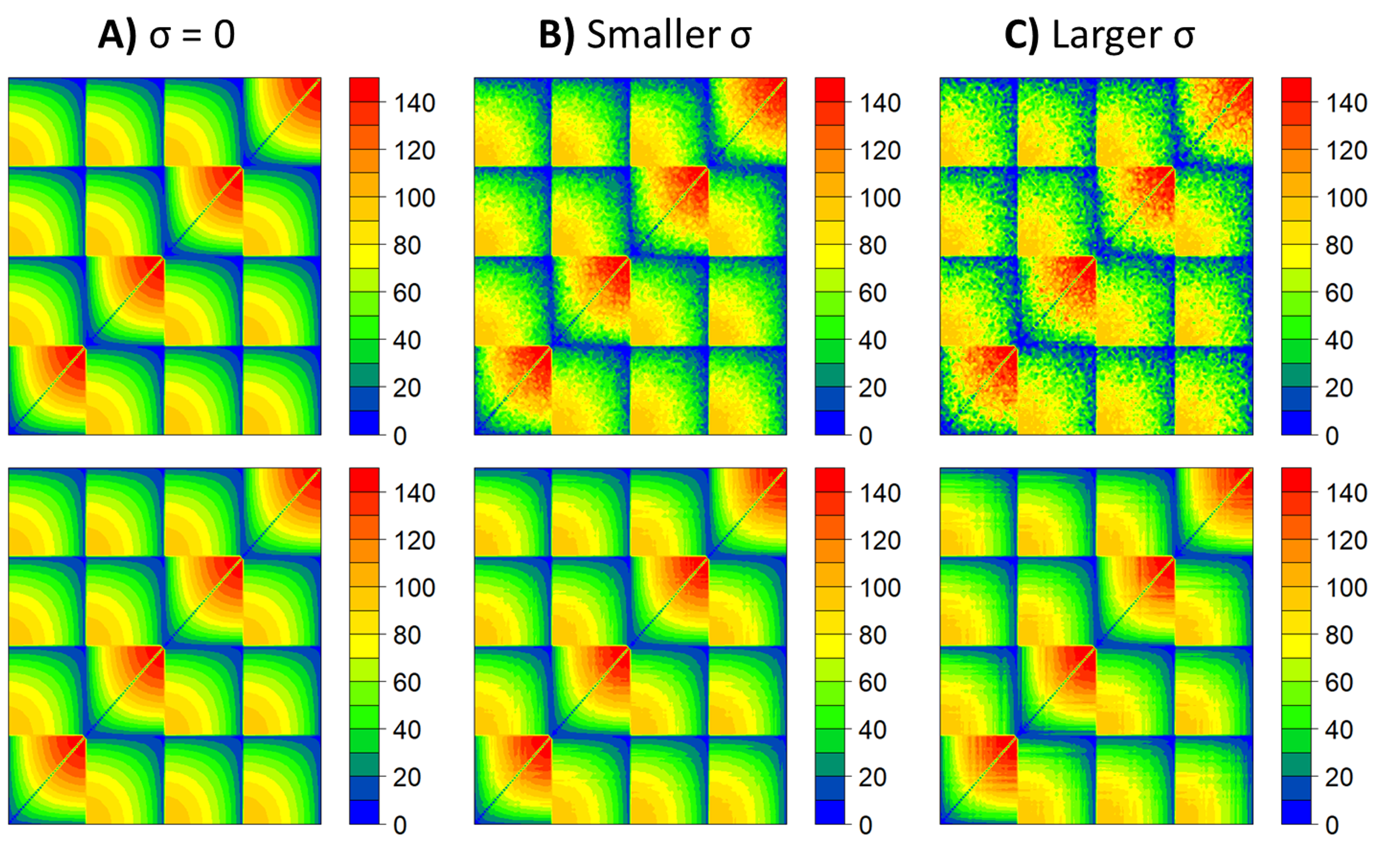}
\caption{A single $H$-Normal NSM with different $\sigma$ values. The estimate of each network is shown below the original network being estimated.}
\label{f:no-noise-NSM}
\end{figure}
Where plot A uses (\ref{e:HNSM}) with $\sigma = 0$, plot B uses $\sigma =.05$ everywhere, leaving the network looking smudged. The estimate looks like a smoothed version of the actual observed network, albeit somewhat ``blurrier" than the underlying network seen in plot A. This performance degradation with increasing $\sigma$ is to be expected. In plot C, $\sigma =.15$ for within community edges, and $\sigma =.2$ for between community edges, and the estimate looks slightly worse than in plot B. 

\subsection{Disassortative network with spurious patterns}
Figure \ref{f:planted} introduces several changes. First, the communities are disassortative, as between community edges are larger than within community edges. Second, the within community edge weights are i.i.d. Third, the between community edge weights are generated using randomly generated Gamma parameters for each node, and using those as inputs into a negative binomial distribution. This is not generated as an $H$-Normal NSM, yet we still use our estimation procedure. Fourth, the nodes are not ordered. If the true community orderings are not known a priori, a network may look like the left plot of Figure \ref{f:planted}. Following estimation, the communities are clustered correctly, and the between community estimate broadly looks smoother than the original network. The initial estimated network appears to amplify spurious structure in the within community edges, giving some order to the pure randomness seen in the original network. Utilizing the procedure discussed in Section \ref{s:spurious}, within community edges for subnetworks with spurious patterns are replaced by the median value of that subnetwork's edge weights. 
\begin{figure}[h!]
 \includegraphics[width=\linewidth]{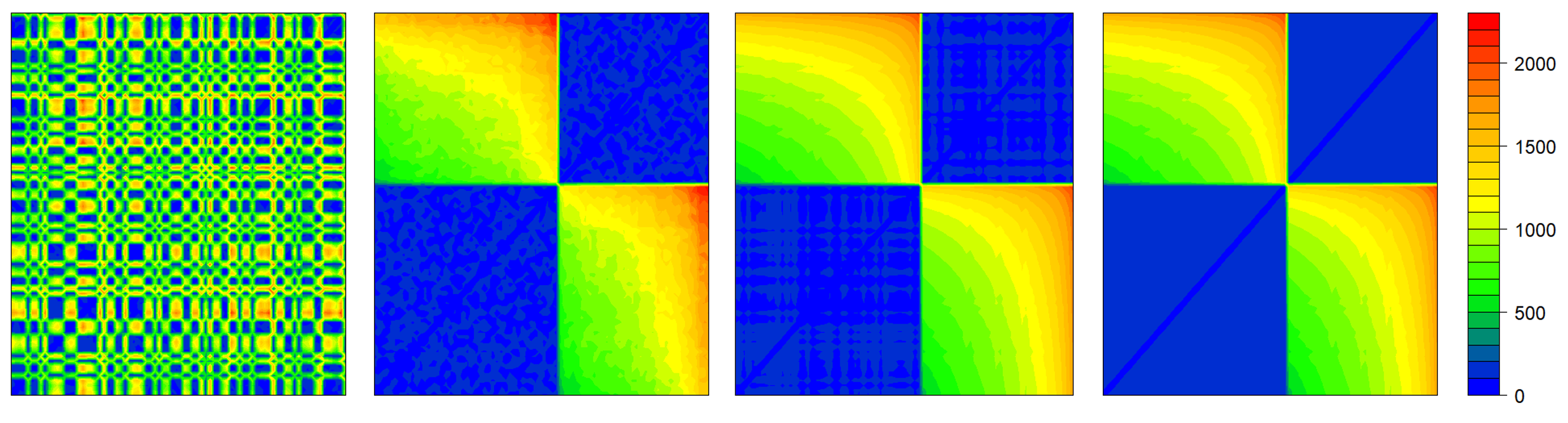}
\caption{Network with spurious pattern of within community edges. From left to right: original, reordered, initial estimate, final estimate.}
\label{f:planted}
\end{figure}
 
\subsection{Missing edges}\label{s:sparse_nets}
Figure \ref{f:missing} shows the network in Figure \ref{f:no-noise-NSM} after deleting many edges at random, along with the final estimated networks using the procedure discussed in Section \ref{s:sparse}. The communities are estimated for the network with 20\% of the edges missing, but are assumed to be known for the network with 75\% of edges missing.  Even accounting for this, as one might expect, the reconstruction is more successful with fewer missing edges.   

\begin{figure}[h!]
 \includegraphics[width=\linewidth]{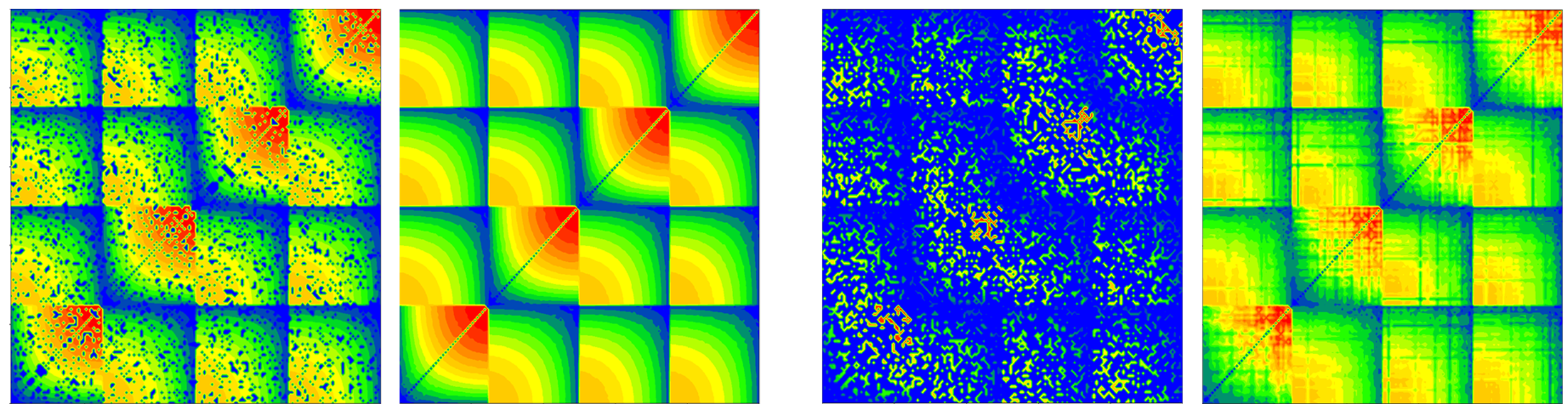}
\caption{Networks with missing edges. Left: network with 20\% of edges deleted and its estimate. Right: network with 75\% of edges deleted and its estimate.}
\label{f:missing}
\end{figure}

\section{Applications}
In this section, we show how the methods described in this paper work on real data where the ground truth clusterings are unknown. In this case, we will show the original network, the network reordered by community then within community degree, and then show the estimated network. 
\subsection{Brain networks}
Figure \ref{f:LDF}, which has already been discussed in Section \ref{s:intro}, shows a preprocessed DTI scan from the ADNI database (http://adni.loni.usc.edu). As in \cite{leinwand2020characterizing}, for this scan, the cortical surface has been parcellated into the 148 regions of the Destrieux Atlas using FreeSurfer on the T1-weighted MRI scan. Then probabilistic fiber tractography was applied on DWI and T1-weighted images using FSL software library to obtain a 148 $\times$ 148 matrix. Each entry in the matrix is the log of the count of white matter fibers connecting two brain regions. 

In contrast to the structural brain network discussed above, Figures \ref{f:ADHD} and \ref{f:typical} show the functional brain networks of subject IDs 293 and 108 from \cite{Brown_2012}, two pre-processed fMRI scans from the ADHD-200 sample.  Both networks come from females, where one is age 10.73 and typically developing, and the other is age 10.81 with ADHD. Both are processed using the Athena pipeline resulting in 190 regions. More details about preprocessing can be found at http://umcd.humanconnectomeproject.org.
 
The most obvious difference between the results is the ADHD network is split into 4 communities, while the typical network breaks into 5 communities. Looking at the typical network, we see clearer negative associations between communities than in the ADHD network, particularly accounting for the slightly different axes. The estimate of both networks shows some ``plaid" looking patterns, as opposed to colors monotonically changing in one direction, which indicates the ordering of $\Psi$ values within communities may not be the same as the ordering of $\Psi$ values between communities. More defined communities and greater negative association patterns in the typical scan would appear to support the hypothesis that ADHD subjects exhibit less modular brain organizations than typical subjects. 

\begin{figure}[h!]
 \includegraphics[width=\linewidth]{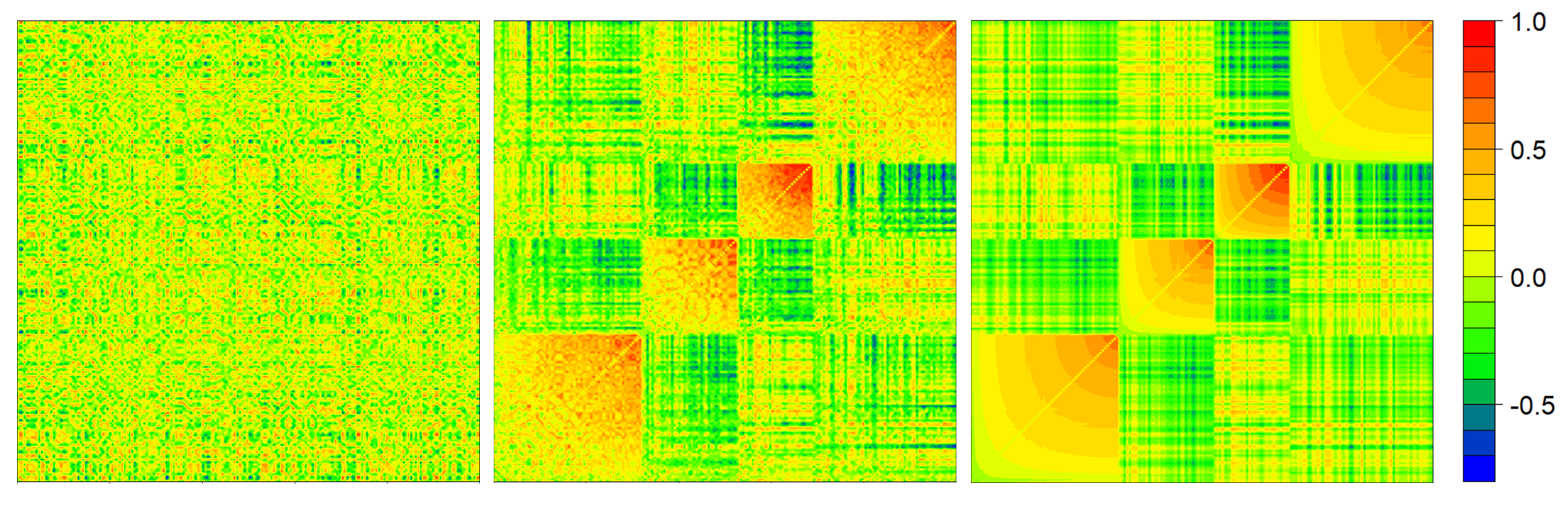}
\caption{ADHD brain network.}
\label{f:ADHD}
\end{figure}
\begin{figure}[h!]
 \includegraphics[width=\linewidth]{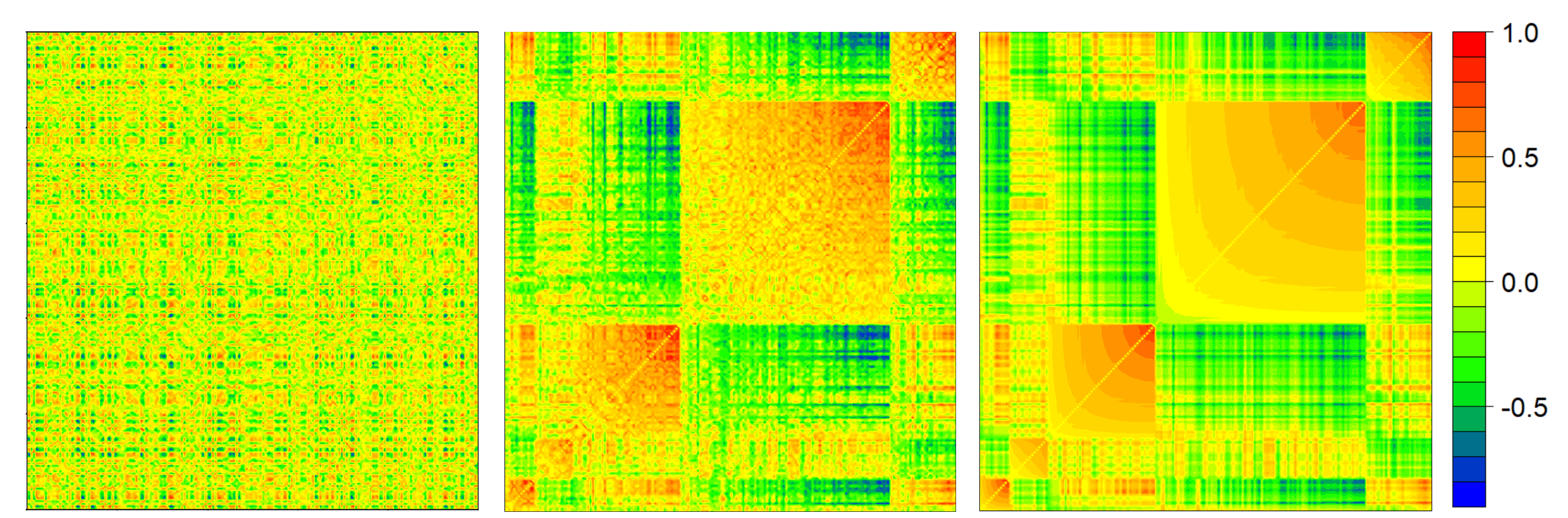}
\caption{Control brain network.}
\label{f:typical}
\end{figure}

\begin{figure}
 \includegraphics[width=\linewidth]{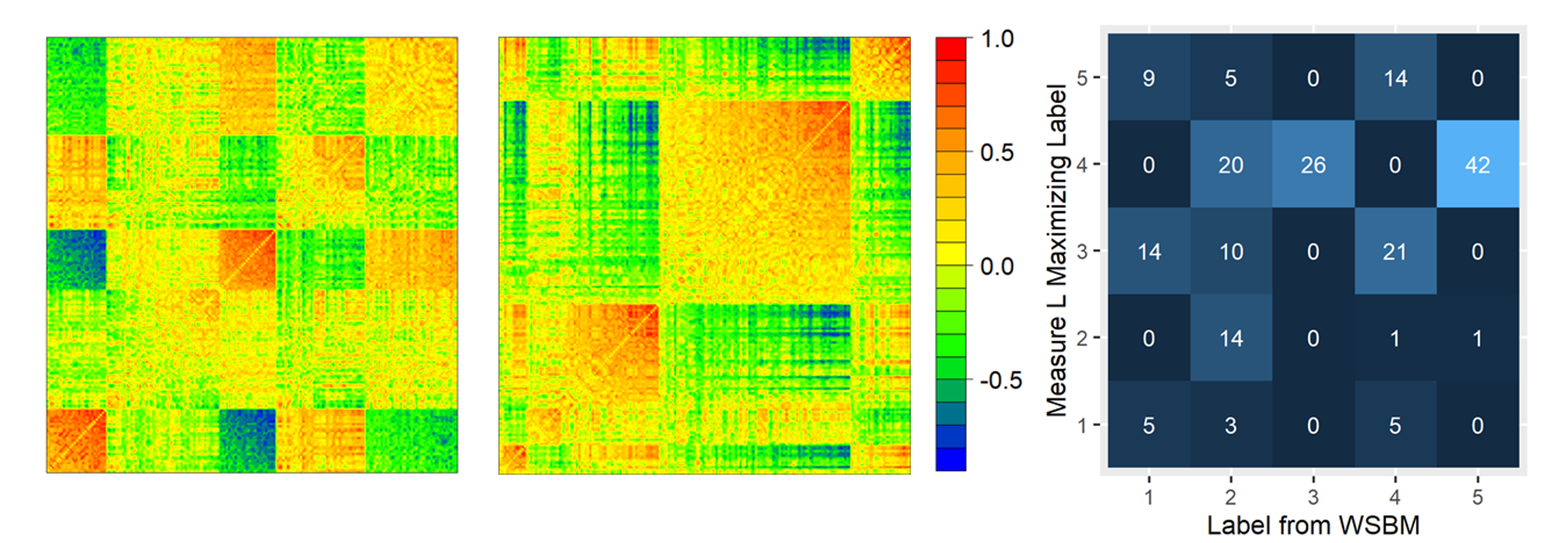}
\caption{From left to right: control brain network reordered by communities estimated using WSBM. The same network reordered by communities estimated using an approximate $L$ maximizing algorithm. The confusion matrix of node labels between these community estimates (note that label values do not reflect ordering).}
\label{f:WSBM_comp}
\end{figure}

Figure \ref{f:WSBM_comp} shows the control network rearranged and estimated based on WSBM community detection, using the Matlab package accompanying \cite{aicher2013adapting} and \cite{aicher2015learning}. Our model does require prespecification of the number of communities nor the distribution of edge weights within or between communities, but for the sake of comparison, we instruct their package to mimic the structure of our results as best as possible, segmenting the network into 5 communities, ignoring the edge distribution, and assuming the weight distribution is Normal. WSBM appears to cluster nodes such that the induced subnetworks have edge weights confined to a relatively narrow band of values, which gives the impression of more solid colors and fewer gradients in the rearranged matrix. It also produces relatively evenly sized clusters. Using an $L$ maximizing algorithm, on the other hand, produces a larger community containing almost half of the nodes. The $L$ maximizing communities produce a measure value of 6580 compared to 3276 for the WSBM communities. Implementing the estimation methods from Section \ref{s:intronmici}, the mean squared error of the final estimated matrix in Figure \ref{f:typical} is .024 compared to .033 if using the WSBM estimates. The WSBM communities also yield larger $\widehat{\sigma}$ values. This provides evidence that the community detection using the measure $L$ is capturing something different than WSBM community detection, and likely a signal more suitable for the estimation methods described in this paper. Further experimentation has shown that WSBM community detection of plot A in Figure \ref{f:no-noise-NSM} does not match the intuitive visual clustering. Functional brain networks may not be organized according to an NSM or LSM, but the presence of detectable negative associations merits further investigation.    

\subsection{State to state migration ``affinity"} 
Taking the state-to-state migration flows data from the 2017 American Community Survey 1-Year Estimates and dividing each cell in that table by the outgoing state's total outflows gives a transition probability matrix for those people who left their state in 2017. For the network shown in Figure \ref{f:migration}, this transition probability matrix is added to its transpose to get a symmetric matrix. This final network ignores the direction of greater inflows or outflows, but instead represents the ``affinity" between the two states in question. The results show geographic communities which appear to give a reasonable segmentation based on geography.  The estimated network displays positively associated within community dynamics, indicating both homophily and degree correction. However, the within community $\widehat{\sigma}$ estimates are relatively large for this network, so these estimates have a relatively narrow range. This may be due to small subnetwork size, but also because some of the largest values lie in the interior of the subnetworks, rather than on the frontier. 
\begin{figure}[h!]
 \includegraphics[width=\linewidth]{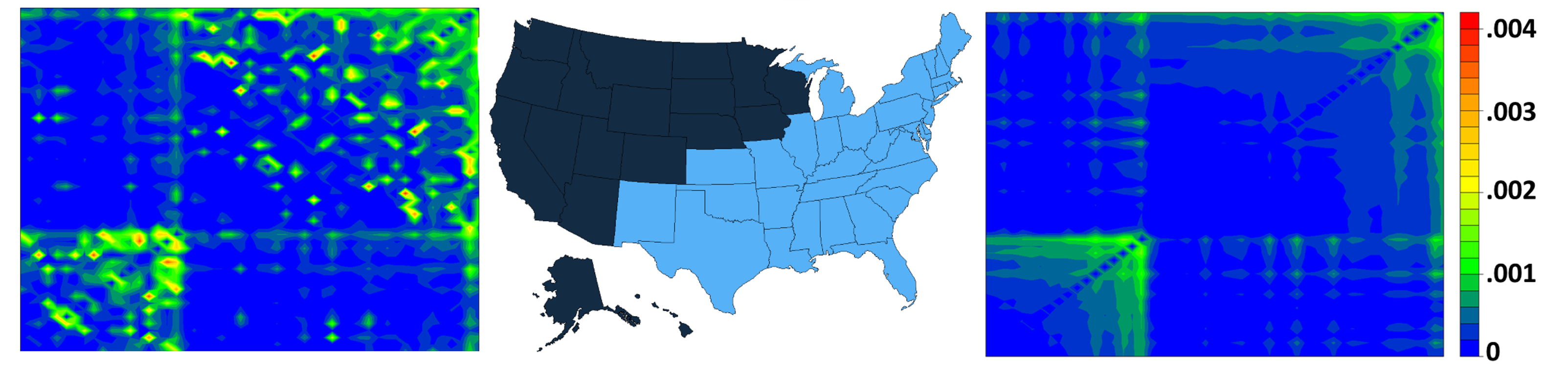}
\caption{2017 state to state migration.}
\label{f:migration}
\end{figure}

\section{Conclusions} 
We have introduced new models for dense weighted networks, wherein edge weights depend on node sociabilities and community memberships. The development of these models spurred estimation techniques for networks of this kind. With minor modifications, these estimation techniques appear to be applicable to an even broader class of networks than those introduced in this paper.  Furthermore, one can use the results from (\ref{e:MSE}) as a gauge of whether the described estimation process is appropriate for a particular network.  

One potential consideration for future work is determining whether $\Psi$ values should be estimated at the global or local level. In a case where each node's $\Psi$ value is globally consistent, estimating it across the whole network would be preferred to estimating several local estimates. However, given potentially different $H$-functions across different subnetworks, pooling this information is not necessarily straightforward. Similarly, a different kind of information pooling may also play an important role for modeling the dynamics of a given network observed repeatedly over time. The introduced models may also lend themselves to extensions for more generalized forms of graphs such as multilayer networks or -- following up on Appendix \ref{s:Hext} -- hypergraphs.
\bibliography{yourbibfile}
\appendix
This extended appendix discusses several topics related to the main paper. Appendix \ref{sa:LSM-proof} contains two concentration inequalities for LSM networks, followed by auxiliary lemmas. Appendix \ref{sa:repeats} discusses estimation while accounting for certain diagonal 0's for within community subnetworks, along with how to adapt the estimation procedures when there are repeated values. Appendix \ref{mod-problems} explains how modularity may fail to characterize the kinds of communities discussed in the main paper. Appendix \ref{s:greedy-alg} introduces a greedy, agglomerative approach to community detection by trying to maximize the measure $L$.  Appendix \ref{s:spatial} considers community detection as a spatial clustering problem, and Appendix \ref{s:spectral-alg} includes an algorithm built upon spectral clustering to try to maximize $L$. Appendix \ref{s:rotated} shows that the real parts of the first several eigenvectors of a normalized version of the network also appear to capture community information. Appendix \ref{more-sims} displays additional simulated networks and results. Finally, Appendix \ref{s:Hext} extends both $H$-functions and the kinds of ``errors" which may be observed in LSMs or NSMs. 

\section{Concentration inequality for Normal LSM}\label{sa:LSM-proof}
In Section \ref{s:intronmici}, we introduced methods for estimating sociability parameters for certain kinds of network generating models. It would be helpful to ensure these estimates are actually capturing the underlying ``error-free" network generating mechanisms, what we refer to as the network's SC. As there are several introduced models, the accuracy of the estimation results may depend on the specific functional forms of a given network. We present below a result on estimation accuracy of the SC for a Normal LSM network (\ref{e:NLSM2}). The estimation procedure is formulated through theoretical means; it remains to be seen how the procedure compares to the practical estimation approach taken in Sections \ref{s:intronmici} and \ref{s:comm-det}. The proof of the result adapts ideas from \cite{noroozi2019estimation}. 

 \subsection{Concentration for Normal LSM with known $\sigma_{max}^2$}
\begin{theorem}\label{T:conc}
Let  $A$ be a Normal LSM network such that every edge weight is generated as in (\ref{e:NLSM2}), and $P_*$ be the network such that each edge weight  has the same generating process as the corresponding edge weight in $A$, but with every $\sigma$ value set to 0. Also let $\hat{P}$ be the estimated network (of the form (\ref{e:NLSM2}) with $\sigma =0$)  induced by the clustering of nodes that minimizes 
\begin{equation} \label{e:frob-and-pen}
||A- \hat{P}||_F^2 + \text{Pen}(n , \hat{K}),
\end{equation}
where $\hat{K}$ is the number of different communities in this ``best" clustering, and 
\begin{equation} \label{e:pen1}
Pen(n, K) = 6 \sigma_{max}^2 \left(C_1 nK + C_2 K^2 \log(n)\right)  + (6C_3 + \frac{2}{c}) \sigma_{max}^2(\log(n)+n\log(K)),
\end{equation}
where $n$ is the number of nodes in the network, $K$ is the number of communities in the clustering, $c$ is some constant such that $0<c<1$, $\sigma_{max}^2 < \infty$ is the largest variance parameter of any generating function for edges in network $A$, and $C_1$ and $C_2$ are as given in Lemma \ref{L:NLSM1}.
Then, for some constant $C_3$ and any $t >0$,
\begin{equation}\label{e:ineq1}
\mathbb{P}\left(||\hat{P} -P_*||_F^2 \le (1-c)^{-1}  Pen(n, K_*)+  \frac{C_3}{c}\sigma_{max}^2t\right) \ge 1-3e^{-t}.
\end{equation}
\end{theorem}
\begin{proof}

It is first useful to define the $\hat{P}$ induced by a particular clustering. Assuming we have clustered all nodes of $A$ into $\hat{K}$ estimated communities, we denote the subnetwork including only edges connecting nodes in estimated community $\hat{k}$ to nodes in estimated community $\hat{l}$ as $A^{(\hat{k}, \hat{l})}$. 
Letting $\mathbbm{1}_n$ be a length $n$ column vector of 1's, define $\Pi(A^{(\hat{k}, \hat{l})}, \hat{Z}^{(\hat{k})}, \hat{Z}^{(\hat{l})})$ as the projection of $A^{(\hat{k}, \hat{l})}$ of the form (\ref{e:NLSM2}) with $\sigma =0$ which minimizes the Frobenius norm to the observed subnetwork, which can be written as 
\begin{equation}\label{e:submat-projection}
\Pi(A^{(\hat{k}, \hat{l})}, \hat{Z}^{(\hat{k})}, \hat{Z}^{(\hat{l})})= \hat{\gamma}^{(\hat{k}, \hat{l})}\mathbbm{1}_{\hat{n}_{\hat{k}}}\mathbbm{1}_{\hat{n}_{\hat{l}}}' + \hat{\alpha}^{(\hat{k}, \hat{l})} \hat{Z}^{(\hat{k}, \hat{l})} \mathbbm{1}_{\hat{n}_{\hat{l}}}' + \hat{\beta}^{(\hat{k}, \hat{l})}\mathbbm{1}_{\hat{n}_{\hat{k}}}\hat{Z}^{(\hat{l}, \hat{k})\prime}, 
\end{equation}
such that
\begin{equation}\label{e:submat-argmin}
\hat{\gamma}^{(\hat{k}, \hat{l})}, \hat{\alpha}^{(\hat{k}, \hat{l})}, \hat{\beta}^{(\hat{k}, \hat{l})}, \hat{Z}^{(\hat{k}, \hat{l})}, \hat{Z}^{(\hat{l}, \hat{k})}  = \argmin {\gamma, \alpha, \beta, Z_1, Z_2} ||A^{(\hat{k}, \hat{l})} - (\gamma\mathbbm{1}_{\hat{n}_{\hat{k}}}\mathbbm{1}_{\hat{n}_{\hat{l}}}' + \alpha Z_1 \mathbbm{1}_{\hat{n}_{\hat{l}}}' + \beta \mathbbm{1}_{\hat{n}_{\hat{k}}}Z_2')||_F. 
\end{equation}
For identifiability purposes, we also mandate the following constraints:
\begin{equation}
\sum \hat{Z}^{(\hat{k}, \hat{l})} = 0, \quad \sum \hat{Z}^{(\hat{l}, \hat{k})}= 0,  \quad SD(\hat{Z}_1) = 1,  \quad SD(\hat{Z}_2) = 1, \quad \hat{\alpha}^{(\hat{k}, \hat{l})} \ge 0,  \quad \hat{\beta}^{(\hat{l}, \hat{k})} \ge 0.
\end{equation}

For this proof, we need not explicitly express every component quantity in (\ref{e:submat-argmin}), but note that $\Pi(A^{(\hat{k}, \hat{l})}, \hat{Z}^{(\hat{k})}, \hat{Z}^{(\hat{l})})$ has rank $\le 3$ by construction. In the between estimated community subnetworks where $\hat{k} \ne \hat{l}$, the estimates $\hat{P}^{(\hat{k}, \hat{l})} = \Pi(A^{(\hat{k}, \hat{l})}, \hat{Z}^{(\hat{k})}, \hat{Z}^{(\hat{l})})$.

However, for within estimated community subnetworks where $\hat{k} = \hat{l}$, we must define $\Pi(A^{(\hat{k}, \hat{k})}, \hat{Z}^{(\hat{k})}, \hat{Z}^{(\hat{k})})$ slightly differently, so as to ignore any influence from certain 0's on the diagonal of $A^{(\hat{k}, \hat{k})}$. In that case, we keep the pertinent constraints while modifying (\ref{e:submat-projection}) and (\ref{e:submat-argmin}) as follows:
\begin{equation}
\Pi(A^{(\hat{k}, \hat{k})}, \hat{Z}^{(\hat{k})}, \hat{Z}^{(\hat{k})})= \hat{\gamma}^{(\hat{k}, \hat{k})}\mathbbm{1}_{\hat{n}_{\hat{k}}}\mathbbm{1}_{\hat{n}_{\hat{k}}}' + \hat{\alpha}^{(\hat{k}, \hat{k})} \hat{Z}^{(\hat{k}, \hat{k})} \mathbbm{1}_{\hat{n}_{\hat{k}}}' + \hat{\alpha}^{(\hat{k}, \hat{k})}\mathbbm{1}_{\hat{n}_{\hat{k}}}\hat{Z}^{(\hat{k}, \hat{k})\prime}, 
\end{equation}
such that
\begin{equation}
\hat{\gamma}^{(\hat{k}, \hat{k})}, \hat{\alpha}^{(\hat{k}, \hat{k})},  \hat{Z}^{(\hat{k}, \hat{k})}  = \argmin {\gamma, \alpha, \beta, Z} ||\big(A^{(\hat{k}, \hat{k})} - (\gamma\mathbbm{1}_{\hat{n}_{\hat{k}}}\mathbbm{1}_{\hat{n}_{\hat{k}}}' + \alpha Z \mathbbm{1}_{\hat{n}_{\hat{k}}}' + \alpha \mathbbm{1}_{\hat{n}_{\hat{k}}}Z')\big)_{u>v}||_F. 
\end{equation}
This is the projection which minimizes the Frobenius norm to the original subnetwork only with respect to off-diagonal entries, i.e. where $u \ne v$. $\Pi(A^{(\hat{k}, \hat{k})}, \hat{Z}^{(\hat{k})}, \hat{Z}^{(\hat{k})})$ also has rank $\le 3$, but this projection does not necessarily have 0's on the diagonal. This gives rise to within community edge estimates, $\hat{P}^{(\hat{k}, \hat{k})}$, which is equal to $\Pi(A^{(\hat{k}, \hat{k})}, \hat{Z}^{(\hat{k})}, \hat{Z}^{(\hat{k})})$ in the off diagonal entries, but replaces the diagonal entries with 0's (which are correct by construction). We also represent this estimate as $\Pi_0(A^{(\hat{k}, \hat{k})}, \hat{Z}^{(\hat{k})}, \hat{Z}^{(\hat{k})})$, where the 0 subscript indicates the diagonals are forced to 0. 
$\hat{P}^{(\hat{k}, \hat{k})}$ is not a projection of $A^{(\hat{k}, \hat{k})}$, but is closer to $A^{(\hat{k}, \hat{k})}$ in Frobenius norm than $\Pi(A^{(\hat{k}, \hat{k})}, \hat{Z}^{(\hat{k})}, \hat{Z}^{(\hat{k})})$, which is itself the best projection of $A^{(\hat{k}, \hat{k})}$ with respect to the Frobenius norm for off diagonal entries.  

The projections defined on the observed network and the estimated communities differ from the projections we define on $P_*$ and the estimated communities. In defining \\$\Pi(P_*^{(\hat{k}, \hat{l})}, \hat{Z}^{(\hat{k})}, \hat{Z}^{(\hat{l})})$, or $\Pi_0(P_*^{(\hat{k}, \hat{k})}, \hat{Z}^{(\hat{k})}, \hat{Z}^{(\hat{k})})$, we treat the $\hat{Z}$ values as fixed, using the estimated values from $\Pi(A^{(\hat{k}, \hat{l})}, \hat{Z}^{(\hat{k})}, \hat{Z}^{(\hat{l})})$, and only allowing $\hat{\gamma}^{(\hat{k}, \hat{l})}, \hat{\alpha}^{(\hat{k}, \hat{l})},$ and $\hat{\beta}^{(\hat{k}, \hat{l})}$ values to vary. 

With these preliminaries in place, we follow the proof of Theorem 1 in \cite{noroozi2019estimation}. By assumption, 
$$||A- \hat{P}||_F^2 + \text{Pen}(n , \hat{K}) \le ||A- P_*||_F^2 + \text{Pen}(n , K_*).$$
Letting 
\begin{equation}\label{Xi-def}
\Xi = A - P_*, 
\end{equation}
writing $Tr(M)$ for the trace of matrix $M$, and assuming the network has been rearranged into blocks by estimated communities, adding and subtracting $P_*$ within $||A-\hat{P}||_F^2$ on the left-hand side gives
\begin{equation}\label{e:inequality-penalty}
||\hat{P} -P_*||_F^2 \le 2Tr(\Xi'(\hat{P} -P_*))  + \text{Pen}(n , K_*) - \text{Pen}(n , \hat{K}).
\end{equation}
Noting $2 Tr(\Xi'(\hat{P} -P_*)) = 2 \sum\limits_{\hat{k}, \hat{l}=1}^{\hat{K}}Tr\left(\Xi^{(\hat{k}, \hat{l})\prime}(\hat{P}^{(\hat{k}, \hat{l})} -P_*^{(\hat{k}, \hat{l})})\right)$, adding and subtracting \\$\Pi_0(P_*^{(\hat{k}, \hat{k})},  \hat{Z}^{(\hat{k})}, \hat{Z}^{(\hat{k})})$ to within estimated community subnetworks and $\Pi(P_*^{(\hat{k}, \hat{l})},  \hat{Z}^{(\hat{k})}, \hat{Z}^{(\hat{l})})$ to between estimated community subnetworks yields for the trace term in (\ref{e:inequality-penalty}):
\begin{align}\label{broken-trace}
&2 \sum\limits_{\hat{k}\ne \hat{l}}Tr\left(\Xi^{(\hat{k}, \hat{l})\prime} \Pi(\Xi^{(\hat{k}, \hat{l})},  \hat{Z}^{(\hat{k})}, \hat{Z}^{(\hat{l})})\right) + 2 \sum\limits_{\hat{k}=1}^{\hat{K}}Tr\left(\Xi^{(\hat{k}, \hat{k})\prime} \Pi_0(\Xi^{(\hat{k}, \hat{k})},  \hat{Z}^{(\hat{k})}, \hat{Z}^{(\hat{k})})\right) & \nonumber \\
&+2 \sum\limits_{\hat{k} \ne \hat{l}}Tr\left(\Xi^{(\hat{k}, \hat{l})\prime} (\Pi(P_*^{(\hat{k}, \hat{l})},  \hat{Z}^{(\hat{k})}, \hat{Z}^{(\hat{l})}) - P_*^{(\hat{k}, \hat{l})})\right) +2 \sum\limits_{\hat{k}=1}^{\hat{K}}Tr\left(\Xi^{(\hat{k}, \hat{k})\prime} (\Pi_0(P_*^{(\hat{k}, \hat{k})},  \hat{Z}^{(\hat{k})}, \hat{Z}^{(\hat{k})}) - P_*^{(\hat{k}, \hat{k})})\right). &
\end{align}
We handle the four sums in (\ref{broken-trace}) separately. Note that in the between community subnetworks, $Tr\left(\Xi^{(\hat{k}, \hat{l})\prime}\Pi(\Xi^{(\hat{k}, \hat{l})},  \hat{Z}^{(\hat{k})}, \hat{Z}^{(\hat{l})})\right) = ||\Pi(\Xi^{(\hat{k}, \hat{l})},  \hat{Z}^{(\hat{k})}, \hat{Z}^{(\hat{l})}))||_F^2$ since $\Pi(\Xi^{(\hat{k}, \hat{l})},  \hat{Z}^{(\hat{k})}, \hat{Z}^{(\hat{l})})$ is a projection. $\Pi(\Xi^{(\hat{k}, \hat{l})},  \hat{Z}^{(\hat{k})}, \hat{Z}^{(\hat{l})})$ has rank $\le 3$ as noted above, so  
\begin{equation}\label{Theorem-6Lemma1}
2||\Pi(\Xi^{(\hat{k}, \hat{l})},  \hat{Z}^{(\hat{k})}, \hat{Z}^{(\hat{l})})||_F^2 \le 6||\Pi(\Xi^{(\hat{k}, \hat{l})},  \hat{Z}^{(\hat{k})}, \hat{Z}^{(\hat{l})})||_{op}^2 \le 6||\Xi^{(\hat{k}, \hat{l})}||_{op}^2,
\end{equation}
where $||M||_{op}$ denotes the usual spectral norm of matrix $M$. The second inequality holds because the spectral norm is sub-multiplicative, and for a projection matrix $P$, $||P||_{op} \le 1$. 

When $\hat{k}=  \hat{l}$, $$Tr\left(\Xi^{(\hat{k}, \hat{k})\prime} \Pi_0(\Xi^{(\hat{k}, \hat{k})},  \hat{Z}^{(\hat{k})}, \hat{Z}^{(\hat{k})})\right) = Tr\left(\Xi^{(\hat{k}, \hat{k})\prime} \Pi(\Xi^{(\hat{k}, \hat{k})},  \hat{Z}^{(\hat{k})}, \hat{Z}^{(\hat{k})})\right),$$
as $Tr(M'N) = \sum\limits_{u,v} M_{uv}N_{uv}$ and the diagonal of $\Xi^{(\hat{k}, \hat{k})}$ is all 0's because there are no self loops in either $A$ or $P_*$.  With this equivalence, the same argument culminating in (\ref{Theorem-6Lemma1}) also holds for within community subnetworks.

To derive a bound on the first two sums of (\ref{broken-trace}), from Lemma \ref{L:NLSM2}, we have 
\begin{equation}\label{e:Lemma1-result}
\mathbb{P}\left(\sum\limits_{\hat{k},\hat{l}=1}^{\hat{K}} ||\Xi^{(\hat{k},\hat{l})}||_{op}^2 \le \sigma_{max}^2 (C_1 n\hat{K} + C_2 \hat{K}^2 \log(n) + C_3 (t+ \log(n) + n \log(\hat{K}))) \right) \ge 1 - e^{-t}.
\end{equation}

Moving to the last two sums in (\ref{broken-trace}), these can be rewritten in terms of the whole network instead of estimated subnetworks. Slightly abusing notation, let $\Pi(P_*, \{\hat{k}\})$ represent the full network matrix where all entries are given by the value dictated by $\Pi(P_*^{(\hat{k}, \hat{l})}, \hat{Z}^{(\hat{k})}, \hat{Z}^{(\hat{l})})$ for between estimated community edges, and by $\Pi_0(P_*^{(\hat{k}, \hat{k})}, \hat{Z}^{(\hat{k})}, \hat{Z}^{(\hat{k})})$ for within estimated community edges. The last two sums in (\ref{broken-trace}) can be represented as  
\begin{equation}\label{intro-Xi-H}
2Tr\left(\Xi'(\Pi(P_*, \{\hat{k}\}) - P_*)\right) = 2||\Pi(P_*,  \{\hat{k}\}) - P_*||_F \, |\langle\Xi, H(\{\hat{k}\})\rangle|,
\end{equation}
where
$$H(\{\hat{k}\}) = \frac{\Pi(P_*, \{\hat{k}\}) - P_*}{||\Pi(P_*, \{\hat{k}\}) - P_*||_F}.$$
Since for any $a, b$ and for $c >0$, $2ab \le c a^2 +\frac{b^2}{c}$,
\begin{equation}\label{second-trace-breakdown}
2Tr\left(\Xi'(\Pi(P_*, \{\hat{k}\}) - P_*)\right) \le c  ||\Pi(P_*, \{\hat{k}\}) - P_*||_F^2 + \frac{ |\langle\Xi, H(\{\hat{k}\})\rangle|^2}{c}.
\end{equation}
Denoting the set of partitions of the nodes into exactly $K$ communities as $\mathcal{G}_K$, for any fixed partition $G \in \mathcal{G}_K$, $\sum_{u, v} (H(G)_{uv})^2 = 1$, and the matrix $\Xi$ consists of independent normally distributed errors with finite variances. Representing the true communities of nodes $u$ and $v$ as $i$ and $j$, respectively, since $|\langle \Xi, H(G)\rangle| = vec(\Xi)'vec(H(G))$, observe that, 
$$ \mathbb{P}\left(|\langle\Xi, H(G)\rangle|^2 >t\right) = \mathbb{P}\left((\sum\limits_{u, v=1}^{n} \sigma_{i j} \epsilon_{uv} H(G)_{uv})^2 >t\right)$$ 
\begin{equation}
= 2\mathbb{P}\left(\sum\limits_{u, v=1}^n \sigma_{i j} \epsilon_{uv} H(G)_{uv} >\sqrt{t}\right) \le 2\mathbb{P}\left(\mathcal{N}(0, \sigma_{max}^2)>\sqrt{t}\right) \le 2e^{-t/2\sigma_{max}^2}.
\end{equation}
Applying the union bound, 
$$\mathbb{P}\left(|\langle \Xi, H(\{\hat{k}\})\rangle |^2 - 2\sigma_{max}^2(\log(n) + n\log(\hat{K})) >2\sigma_{max}^2t\right) $$
$$\le \mathbb{P}\left(\max\limits_{1 \le K \le N} \max\limits_{G \in \mathcal{G}_K} \left(|\langle \Xi, H(G)\rangle |^2 - 2\sigma_{max}^2(\log(n) + n\log(K))\right) >2\sigma_{max}^2t \right) $$
\begin{equation}\label{Xi-H-union}
\le 2nK^n e^{-2\sigma_{max}^2(\log(n) + n\log(K)+t)/ 2\sigma_{max}^2} = 2e^{-t}.
\end{equation}
The squared Frobenius norm matrix $||\Pi(P_*, \{\hat{k}\}) - P_*||_F^2$ in (\ref{second-trace-breakdown}) can be written as 
\begin{equation}
\sum \limits_{\hat{k}\ne \hat{l}} ||\Pi(P_*^{(\hat{k}, \hat{l})},  \hat{Z}^{(\hat{k})}, \hat{Z}^{(\hat{l})}) - P_*^{(\hat{k}, \hat{l})}||_F^2 + \sum \limits_{\hat{k}=1}^{\hat{K}} ||\Pi_0(P_*^{(\hat{k}, \hat{k})},  \hat{Z}^{(\hat{k})}, \hat{Z}^{(\hat{k})}) - P_*^{(\hat{k}, \hat{k})}||_F^2.
\end{equation}
When $\hat{k} \ne \hat{l}$, $||\Pi(P_*^{(\hat{k}, \hat{l})}, \hat{Z}^{(\hat{k})}, \hat{Z}^{(\hat{l})}) - P_*^{(\hat{k}, \hat{l})}||_F \le ||\hat{P}^{(\hat{k}, \hat{l})} - P_*^{(\hat{k}, \hat{l})}||_F$ because $\Pi(P_*^{(\hat{k}, \hat{l})}, \hat{Z}^{(\hat{k})}, \hat{Z}^{(\hat{l})})$ is the \textit{best} projection of $P_*^{(\hat{k}, \hat{l})}$ onto the estimated $\hat{Z}^{(\hat{k}, \hat{l})}$ and $\hat{Z}^{(\hat{l}, \hat{k})}$ values. When $\hat{k} = \hat{l}$, $\Pi(P_*^{(\hat{k}, \hat{k})}, \hat{Z}^{(\hat{k})}, \hat{Z}^{(\hat{k})})$ is also the best projection of $P_*^{(\hat{k}, \hat{k})}$ onto the $\hat{Z}^{(\hat{k}, \hat{k})}$ values with respect to only the off diagonal entries, and the diagonal entries of both $\Pi_0(P_*^{(\hat{k}, \hat{k})}, \hat{Z}^{(\hat{k})}, \hat{Z}^{(\hat{k})})$ and $P_*^{(\hat{k}, \hat{k})}$ are all 0. Therefore,  $||\Pi_0(P_*^{(\hat{k}, \hat{k})}, \hat{Z}^{(\hat{k})}, \hat{Z}^{(\hat{k})}) - P_*^{(\hat{k}, \hat{k})}||_F \le ||\hat{P}^{(\hat{k}, \hat{k})} - P_*^{(\hat{k}, \hat{k})}||_F$. Combining this point with those given in (\ref{intro-Xi-H})--(\ref{Xi-H-union}), 
\begin{equation}\label{second-trace-ineq}
\mathbb{P}\left(2Tr\left(\Xi'(\Pi(P_*, \{\hat{k}\}) - P_*)\right) \le c ||\hat{P} - P_*||_F^2  + \frac{2\sigma_{max}^2(\log(n) +n\log(\hat{K})+ t)}{c}\right)  \ge 1-2e^{-t}.
\end{equation}

Let $\Omega$ denote the set on which, for a particular $t$, both events in (\ref{e:Lemma1-result}) and (\ref{second-trace-ineq}) occur. Then $P(\Omega) \ge 1-3e^{-t}$. Therefore on $\Omega$, by using (\ref{e:inequality-penalty}),
\begin{align}\label{on-omega-result}
||\hat{P} -P_*||_F^2 \le & 6 \sigma_{max}^2 (C_1 n\hat{K} + C_2 \hat{K}^2 \log(n) + C_3(\log(n) + n\log(\hat{K}) + t)) \nonumber \\
&+ c ||\hat{P} - P_*||_F^2  + \frac{2\sigma_{max}^2(\log(n) + n\log(\hat{K}) +t)}{c} + \text{Pen}(n , K_*) - \text{Pen}(n , \hat{K}) .
\end{align}
Letting $0<c < 1$ and $Pen(n, K)$ as in (\ref{e:pen1}), we derive (\ref{e:ineq1}), that is $$\mathbb{P}\big(||\hat{P} -P_*||_F^2 \le (1-c)^{-1}  Pen(n, K_*)+  \frac{C_3}{c}\sigma_{max}^2t\big) \ge 1-3e^{-t}.  \QED$$ \hfill
\end{proof}

This theorem indicates that in the Normal LSM setting, by choosing an appropriate penalty, one can choose a clustering and the number of clusters such  that, with high probability, $||\hat{P}-P||_F^2$, the squared Frobenius distance between the true generating process and the estimate thereof induced by this ``best" clustering, is bounded above by a multiple of the penalty using the true number of communities, accounting for the scale of the $\sigma$ values in the network. As might be expected, in a network with many nodes, many communities, and large $\sigma$ values, the total distance between the estimate and generating process for the network enforced by this bound can become large. However, one observes that by dividing by $n^2$,  the distance from the estimate to the truth continues to shrink on a per edge basis. If $K/n \rightarrow 0$, the estimates are consistent. One drawback is this result relies on employing a penalty which includes $\sigma_{max}^2$, a value which cannot in all cases be assumed to be known. 
    
\subsection{Case with unknown $\sigma^2$}
In practice, we do not know the value of $\sigma_{max}^2$. If instead we assume that throughout the whole network, there is a single value of $\sigma^2$, we can also estimate that value using the data.  Furthermore, letting $\hat{K} =n$ and putting each node into its own community, the estimate $\hat{P}$ will appear to be error free. For any network of interest, this is overfitting, so it is reasonable to set some restriction on the number of communities relative to the number of nodes. The following theorem extends Theorem \ref{T:conc} under these new assumptions, where we do not require prior knowledge of $\sigma_{max}^2$.
\begin{theorem} In the same setting as Theorem \ref{T:conc}, assume the choice of $\hat{K}$ is restricted such that
\begin{equation}\label{community-limit}
C_1 n\hat{K} + C_2 \hat{K}^2 \log(n) + C_3(\log(n)+n\log(\hat{K})) \le \frac{n(n-1)}{4},
\end{equation}
where the constants $C_1, C_2, C_3$ are derived in the same manner as in Theorem \ref{T:conc}, but take on slightly different values here. Additionally, assume network $A$ has a constant variance parameter $\sigma^2$. Let $\hat{P}$ be the estimated network induced by the clustering of nodes that minimizes
\begin{equation}\label{e:frob-and-pen2}
||A- \hat{P}||_F^2 + \hat{\sigma}^2 \text{Pen}(n , \hat{K}),
\end{equation}
where 
\begin{equation}\label{proof2-sigma-hat}
\hat{\sigma}^2 = \frac{||A-\hat{P}||_F^2}{n(n-1)},
\end{equation}
and the penalty is given by 
\begin{equation} \label{e:pen2}
Pen(n, K) = 4(C_1 nK + C_2 K^2 \log(n) + C_3(\log(n)+n\log(K))).
\end{equation}
Then, for $t >0, \epsilon \in [0, 1/2)$, 
\begin{equation}\label{proof2-result}
\mathbb{P}\left(||P_*- \hat{P}||_F^2 \le \frac{\sigma^2}{1-c}\left(C t + (1+\epsilon) \text{Pen}(n , K_*)\right) \right) \ge 1 - 3e^{-t} -  e^{-\frac{3}{32}\epsilon^2 n(n-1)}. 
\end{equation}
\end{theorem}
\begin{proof}    
By assumption,  
$$||A- \hat{P}||_F^2 + \hat{\sigma}^2 \text{Pen}(n , \hat{K}) \le ||A- P_*||_F^2 + \sigma_*^2 \text{Pen}(n , K_*),$$
where $\hat{\sigma}^2 = \frac{||A-\hat{P}||_F^2}{n(n-1)}$ as in (\ref{proof2-sigma-hat}), and $\sigma_*^2 = \frac{||A-P_*||_F^2}{n(n-1)}$. The above is equivalent to
$$||A- \hat{P}||_F^2 \left(1+\frac{\text{Pen}(n , \hat{K})}{n(n-1)}\right) \le ||A- P_*||_F^2 \left( 1+ \frac{\text{Pen}(n , K_*)}{n(n-1)}\right).$$ Then, following the proof of Theorem \ref{T:conc} leading to (\ref{e:inequality-penalty}), we have\\
\begin{align}
&||P_*- \hat{P}||_F^2 \left(1+\frac{\text{Pen}(n , \hat{K})}{n(n-1)}\right)& \nonumber \\
&\le  2Tr(\Xi'(\hat{P} -P_*)) \left(1+\frac{\text{Pen}(n , \hat{K})}{n(n-1)}\right) + \frac{||A- P_*||_F^2}{n(n-1)}\left(\text{Pen}(n , K_*)- \text{Pen}(n , \hat{K})\right). \nonumber
\end{align}
Dividing both sides by $1+\frac{\text{Pen}(n , \hat{K})}{n(n-1)}$ and following the arguments culminating in (\ref{on-omega-result}) but adjusting constants as necessary, with probability $\ge 1-3e^{-t}$:
\begin{align}
&(1-c)||P_*- \hat{P}||_F^2 & \nonumber \\
&\le \sigma^2 (C_1 n\hat{K} + C_2 \hat{K}^2 \log(n) + C_3 (\log(n)+n\log(\hat{K}))+ Ct) + \frac{\sigma_*^2 \text{Pen}(n , K_*)}{1 + \frac{\text{Pen}(n , \hat{K})}{n(n-1)}} - \frac{\sigma_*^2 \text{Pen}(n , \hat{K})}{1 + \frac{\text{Pen}(n , \hat{K})}{n(n-1)}}& \nonumber
\end{align}
\begin{equation}\label{drop-denom}
\le \sigma^2 (C_1 n\hat{K} + C_2 \hat{K}^2 \log(n) +  C_3 (\log(n)+n\log(\hat{K})) + Ct) + \sigma_*^2 \text{Pen}(n , K_*)- \frac{\sigma_*^2 \text{Pen}(n , \hat{K})}{1 + \frac{\text{Pen}(n , \hat{K})}{n(n-1)}},
\end{equation}
where the denominator from the second to last term was dropped in the last inequality. 

Note that $\sigma_*^2$ is the average of $\frac{n(n-1)}{2}$ squared $\mathcal{N}(0, \sigma^2)$ random variables, so its distribution is  $\frac{2\sigma^2}{n(n-1)} \chi_{\frac{n(n-1)}{2}}^2$. Using a result from \cite{johnstone2001chi}, for $\epsilon \in [0, 1/2)$,  
\begin{equation}\label{sigma-vs-sigma}
\mathbb{P}\big((1-\epsilon)\sigma^2 \le \sigma_*^2 \le (1+\epsilon)\sigma^2\big)  \ge 1- e^{-\frac{3}{32}n(n-1)\epsilon^2}.
\end{equation}
Putting together (\ref{drop-denom}) and (\ref{sigma-vs-sigma}), for $\epsilon \in [0, .5)$,  with probability $\ge 1- 3e^{-t} - e^{-\frac{3}{32}\epsilon^2 n(n-1)}$, 
\begin{align}\label{pre-improper}
&(1-c)||P_*- \hat{P}||_F^2 \le C\sigma^2 t + (1+\epsilon)\sigma^2 \text{Pen}(n , K_*) \nonumber \\ 
&+\sigma_*^2\left(\frac{C_1}{1-\epsilon} n\hat{K} + \frac{C_2}{1-\epsilon} \hat{K}^2 \log(n) +\frac{C_3}{1-\epsilon}(\log(n)+n\log(\hat{K})) - \frac{\text{Pen}(n , \hat{K})}{1 + \frac{\text{Pen}(n , \hat{K})}{n(n-1)}}\right).
\end{align}
Using (\ref{community-limit}), (\ref{e:pen2}), and $\epsilon < 1/2$, the term in parentheses in (\ref{pre-improper}) is bounded by
$$2(C_1n\hat{K} + C_2\hat{K}^2 \log(n) +C_3(\log(n)+n\log(\hat{K}))) - \frac{\text{Pen}(n , \hat{K})}{2}\le 0.$$
Hence, with probability $\ge 1- 3e^{-t} - e^{-\frac{3}{32}n(n-1)\epsilon^2}$, $$||P_*- \hat{P}||_F^2 \le \frac{C\sigma^2 t + (1+\epsilon)\sigma^2 \text{Pen}(n , K_*)}{(1-c)}.\qquad \qquad \QED$$ \hfill
\end{proof}

\subsection{Auxiliary results}
The following results were used in the proof of Theorem \ref{T:conc}. 

\begin{lemma} \label{L:NLSM1} 
Let $\Xi$ be a symmetric $n\times n$ matrix with 0's on the diagonal and independent $\mathcal{N}(0, \sigma_{uv}^2)$ entries above the diagonal, where all $\sigma_{uv}^2 \le \sigma_{max}^2$. Let $\Xi$ be partitioned into $K^2$ submatrices $\Xi^{(k, l)}, k, l = 1, \ldots, K$ 
For some constants $C_1, C_2$, and $t>0$, 
\begin{equation}\label{Lemma1}
\mathbb{P}\left(\sum\limits_{k, l =1}^{K} ||\Xi^{(k, l)}||_{op}^2 \le \sigma_{max}^2 (C_1 nK + C_2 K^2 \log(n) + C_3 t) \right) \ge 1 - e^{-t}.
\end{equation}
\end{lemma}
\begin{proof}For a fixed partition, let $\xi$ and $\mu$ be vectors with entries $\xi_{k, l} = ||\Xi^{(k,l)}||_{op} $ and  $\mu_{k,l} = \mathbb{E}||\Xi^{k,l}||_{op}$ and $\eta = \xi - \mu$. Then,
$$ \sum\limits_{k,l=1}^{K} ||\Xi^{(k, l)}||_{op}^2 = ||\xi||^2 \le 2||\eta||^2 + 2||\mu||^2.$$
We start by bounding $||\mu||^2$. Using Theorems 1.1 and 3.1 from \cite{Bandeira_2016}, letting $n_{k}$ and $n_{l}$ denote the number of rows and columns respectively in $\Xi^{(k, l)}$,
 $$\mu_{k, l} = \mathbb{E}||\Xi^{(k, l)}||_{op} \le (1+\epsilon)\left(\sqrt{n_{k}}\sigma_{max} + \sqrt{n_{l}}\sigma_{max} + \frac{6}{\sqrt{\log(1+\epsilon)}}\sigma_{max}\sqrt{\log(\min(n_{k}, n_{l}))}\right)$$ for any $0< \epsilon \le  1/2$. In other words, 
$$\mu_{k,l} \le C_0 \sigma_{max}(\sqrt{n_{k}} + \sqrt{n_{l}} +\sqrt{\log(\min(n_{k}, n_{l}))}\, ),$$
$$\mu_{k,l}^2 \le 3C_0^2 \sigma^2_{max}(n_{k} + n_{l} +\log(\min(n_{k}, n_{l}))),$$
\begin{equation}\label{L:mu}
||\mu||^2 \le 3 C_0^2 \sigma^2_{max}\sum\limits_{k,l=1}^{K} (n_{k} + n_{l} +\log(\min(n_{k}, n_{l}))) \le 6C_0^2 \sigma^2_{max}nK + 3C_0^2 \sigma^2_{max}K^2 \log(n).
\end{equation}
Next, for $1\le k \le l \le K$, $\eta_{k,l} = \xi_{k,l} - \mu_{k, l}$ are all independent random variables since all errors are assumed to be independent. By Theorem 5.8 of \cite{boucheron2013concentration}, $$\mathbb{P}(|\eta_{k, l}| \ge t ) = \mathbb{P}(|\xi_{k,l} - \mu_{k,l}| \ge t ) \le  2e^{{\frac{-t^2}{4\sigma_{max}^2}}},$$ so $\eta_{k,l}$ is sub-gaussian. Since $\mathbb{E}(\eta_{k,l})=0$, using sub-gaussianity, from Proposition 2.5.2 of \cite{vershynin2018high}, there exists a constant $C \le 288e$ such that $$\mathbb{E}(e^{t\eta_{k,l}}) \le e^{\frac{C\sigma_{max}^2 t^2}{2}}.$$
Let $\tilde{\eta}$ be the sub-vector of $\eta$ which includes the $\eta_{k,l}$ values for $1\le k \le l \le K$. Then, Theorem 2.1 of \cite{Hsu_Kakade_Zhang_2012} ensures that, for any square matrix $M$, using the same constant $C$,

$$\mathbb{P}\left(||M\tilde{\eta}||^2 > C\sigma_{max}^2(Tr(M'M) + 2\sqrt{Tr((M'M)^2)t} + 2||M'M||_{op} t )\right) \le e^{-t}.$$
Letting $M=I_{K(K+1)/2}$, this becomes
$$\mathbb{P}\left(||\tilde{\eta}||^2 \ge C\sigma_{max}^2(K(K+1)/2 + \sqrt{2K(K+1)t} +2t\right) \le e^{-t},$$
and since $||\eta||^2 \le 2||\tilde{\eta}||^2$, 
\begin{equation}\label{L:eta}
\mathbb{P}\left(||\eta||^2 \le 2C\sigma_{max}^2K(K+1) + 6 C\sigma_{max}^2 t \right) \ge 1 - e^{-t}.
\end{equation}
Combining (\ref{L:mu}) and (\ref{L:eta}), 
$$\mathbb{P}\left( |\xi||^2 \le 12C_0^2 \sigma^2_{max}nK + 6C_0^2 \sigma^2_{max}K^2 \log(n) + 4C \sigma_{max}^2 K(K+1) + 12 C \sigma_{max}^2  t\right)\ge 1 - e^{-t}.$$
Collecting terms yields (\ref{Lemma1}). \hfill \QED
\end{proof}\\

The following lemma is nearly identical to Lemma 6 in \cite{noroozi2019estimation}, and serves the same purpose, to translate the bound in Lemma \ref{L:NLSM1}, which is conditional on a particular partition, into an unconditional bound. 
\begin{lemma}\label{L:NLSM2}
For any t $> 0$, 
\begin{equation}
\mathbb{P}\left(\sum\limits_{\hat{k},\hat{l}=1}^{\hat{K}} ||\Xi^{(\hat{k},\hat{l})}||_{op}^2 \le \sigma_{max}^2 (C_1 n\hat{K} + C_2 \hat{K}^2 \log(n) + C_3 (t+ \log(n) + n \log(\hat{K}))) \right) \ge 1 - e^{-t}.
\end{equation}
\end{lemma}
\begin{proof}
Denoting the set of partitions of the nodes into $K$ communities as $\mathcal{G}_K$, for any fixed partition $G \in \mathcal{G}_K$, from Lemma \ref{L:NLSM1},
$$\mathbb{P}\left(\sum\limits_{k, l =1}^{K} ||\Xi^{(k, l)}||_{op}^2 \ge \sigma_{max}^2 (C_1 nK + C_2 K^2 \log(n) + C_3 x) \right) \le  e^{-x}.$$
Taking a union bound over all possible partitions and setting $x = t+ \log(n) + n \log(K)$ 
$$\mathbb{P}\left(\sum\limits_{\hat{k},\hat{l}=1}^{\hat{K}} ||\Xi^{(\hat{k},\hat{l})}||_{op}^2 -  \sigma_{max}^2 (C_1 n\hat{K} + C_2 \hat{K}^2 \log(n) + C_3 (t + \log(n) + n \log(\hat{K}))) \ge 0 \right)$$
$$ \le \mathbb{P} \left( \max \limits_{1\le K \le n}\max \limits_{G \in\mathcal{G}_K } \sum\limits_{k, l =1}^{K} ||\Xi^{(k, l)}||_{op}^2 -  \sigma_{max}^2 (C_1 nK + C_2 K^2 \log(n) + C_3 (\log(n) + n \log(K))) \ge \sigma_{max}^2 C_3 t  \right)$$
$$\le \sum \limits_{K=1}^n \sum \limits_{G \in\mathcal{G}_K } \mathbb{P} \left(\sum\limits_{k, l =1}^{K} ||\Xi^{(k, l)}||_{op}^2 - \sigma_{max}^2 (C_1 nK + C_2 K^2 \log(n) + C_3 (\log n + n \log(K)))\ge\sigma_{max}^2 C_3 t \right) $$
$$\le nK^n e^{-t -\log(n) -n\log(K)} = e^{-t}.  \QED$$  \hfill
\end{proof}

\section{Additional estimation details}\label{sa:repeats}
\subsection{Pre-estimating the diagonal for within community NMF}
When applying the methodology described in Section \ref{s:nmicim1} to a set of edges within the same community, symmetry ensures that $a= b$, so there are no conflicts among estimators. There is, however, a need to resolve the fact that no self loops ($W_{uu} =0$) may impact the estimates of $a$ and $\alpha$ for the within community setting $i=j$. To avoid this impact, the values of the diagonal entries $W_{uu}$ should be imputed before computing the NMF. Were self loops allowed, the expected value a self loop would be $\mathbb{E}(W_{uu}|Z) = 2 \alpha Z_u$, so we can estimate this value while accounting for all diagonal zeros and estimates of other $Z$'s by setting
\begin{equation}\label{e:self-loop}
W_{uu} = \frac{\big(2\sum\limits_{v \in i , u \ne v } W_{uv}\big) - \big(\frac{1}{n_i-2} \sum\limits_{v \in i, u \ne v } \sum\limits_{q\in i, q \ne u, v} W_{vq}\big)}{n_i-1}\,,
\end{equation}
where $n_i$ is the number of nodes in community $i$. Essentially, this takes twice the average within community edge weight for a particular node $u$ and subtracts out twice the average non-diagonal edge weight connecting the other nodes in community $i$. This approximates twice the impact node $u$ while negating the collective impact of the other nodes. 

\subsection{Handling repeated edge weight values}
In Section \ref{s:nmcomp}, when estimating the uniformly distributed values of $\widehat{G}(w)$ in (\ref{e:edge-ECDF}) and $\widehat{\Psi}_u^{(j)}$ in (\ref{e:psi-est}), we make an adjustment if we observe the same value multiple times. Let $W_1 < W_2$ be two consecutive sorted values of $W_{uv}$'s in the same pair of communities, with $\widehat{G}(W_1) = \frac{k}{n+1}$ and $\widehat{G}(W_2) = \frac{k+m}{n+1}$, where $m>1$. That is, there are $m$ different edges in the set $\{W_{uv}: \, u \in i, \, v \in j, \, W_{uv} =W_2\}$. Then, we set  
\begin{equation}\label{e:reposition}
\widehat{G}(W_2) = \frac{k + \frac{m}{2} + \frac{1}{2m}}{n+1}.
\end{equation}
The purpose of (\ref{e:reposition}) is to make all the new $\{\widehat{G}(W_{uv}): \, u \in i, \, v \in j, \, W_{uv} =W_2\}$ values to be at least halfway between $\widehat{G}(W_{1})$ and the original $\widehat{G}(W_2)$ value, but where more duplicate values will bring this value down further. 
When there are no duplicate values (i.e. $W_1 < W_2 < W_3<\ldots<W_{n}$), this formula leaves $\widehat{G}(W_2)$ intact. When duplicate values are present, there will still be duplicate values after applying (\ref{e:reposition}), but they are moved to a different location between $\frac{k}{n+1}$ and $\frac{k+m}{n+1}$.  

 If the $D_{j}(u)$ values repeat in (\ref{e:local-degree}), we make an analogous adjustment to (\ref{e:psi-est}).

\section{Additional community detection details}\label{a:comm-det2}
In this appendix, more information is provided about community detection for the kinds of networks described in the main paper, including motivating ideas, algorithms, and some discussion of why existing community detection methods may be inappropriate in this setting.

\subsection{Limitations of modularity}\label{mod-problems}
Many clustering algorithms on networks attempt to maximize modularity, but this may not be the appropriate measure to use for dense networks. To see this, consider the network shown in Figure \ref{f:reconstruction}. That figure begins on the left with a network where $\sigma =0$ with 74 nodes in 2 communities exhibiting within community positive association and between community negative association. Within community edges are drawn from a uniform distribution with a maximum of 150, while between community edges are drawn from a uniform distribution with a maximum of 100. Node sociability parameters for both communities range from .05 to .95 in increments of .025. The $H$-function in the original network is the same as that used for the rightmost plot in Figure \ref{f:different-h-funcs}, though since the between community edges have negative association, the inputs to that $H$-function are $1 - \Psi_u$ and $1 - \Psi_v$. While the network is not strictly assortative, after accounting for node sociabilities and using the appropriate inputs, edges between nodes in the same communities have 50\% greater weights than edges between nodes in different communities. In this sense there is some notion of homophily that is absent in other regimes where modularity fails. Ordered as in the figure, one can visually identify 2 distinct communities. The second plot from the left in the figure is the estimate of the left plot using the community assignments when clustering nodes using the walktrap algorithm of \cite{pons2006computing} with four steps, which returns three communities, not two. The third plot is the estimate of the first plot using the community assignments by calculating the leading non-negative eigenvector of the modularity matrix of the graph. The fourth plot is the estimate of the first plot using the correct assignments. The modularity of the true communities on this network is in fact slightly negative. This failure of community detection algorithms will feed incorrect community labels to the estimation procedure, leading to estimates that don't preserve the structure of the original network, as seen in the figure. 
\setcounter{figure}{9} 
\begin{figure}
\begin{center}
 \includegraphics[width = \linewidth]{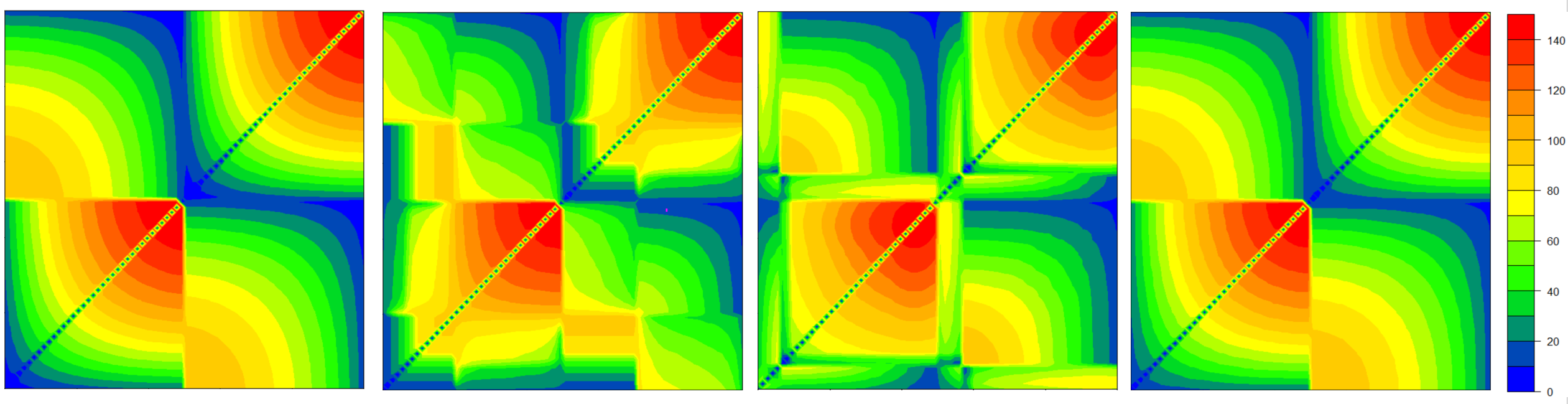}
\end{center}
\caption{Reconstructing a network (left) based on different community assignments.}
\label{f:reconstruction}
\end{figure}

The issue is modularity tries to identify highly interconnected nodes, where edge weights within communities are expected to be higher than edge weights between communities. The methodology described above needs clusters to have a different property to work appropriately, namely that the nodes in each community should obey a kind of monotonicity. Nodes in the same community should have similar patterns of connecting to nodes in other communities, and their own community. In the network in Figure \ref{f:reconstruction}, all nodes in community 1 ``prefer" other nodes in community 1 with high sociability ($\Psi$) values, but ``prefer" nodes in community 2 with low sociability values. This shared ordering of preferences over nodes in each community is crucial for ensuring that the estimation procedure will get appropriate orderings of local sociability statistics. 

It is not difficult to construct networks where it would be useful to combine the community detection using $L$ with modularity. For example, assume communities $i$ and $j$ have positive association both within community and between the communities, but the within community edges are generally much larger than the between community edges. Also assume both $i$ and $j$ have negative association with community $k$. Clustering nodes to simply share ordering of preferences would separate nodes in community $k$ but would not distinguish nodes in community $i$ from nodes in community $j$. Subsequently employing a modularity maximization algorithm on the estimated community consisting of nodes in $i$ and $j$ would recover the true communities and lead to better estimates for all of the edges within and between $i$ and $j$.

\subsection{Greedy algorithm for community detection} \label{s:greedy-alg}
A direct algorithm for maximizing $L$ in (\ref{e:measure}) is to try to iteratively combine nodes into communities which will greedily make $L$ larger. It is computationally impractical to test every possible partition of nodes to maximize $L$. However, based on the structure of $L$, Algorithm \ref{alg:community-greedy} tries to combine communities which are most correlated with one another. 

At the start, each node is placed into its own estimated community. The ``aggregate degree" of each estimated community is defined as the total weight of the edges emanating from any node in the estimated community, where edges are double counted if they connect two nodes in the same estimated community. In Algorithm \ref{alg:community-greedy}, the aggregate degree of an estimated community is equal to the column sum of the estimated community's corresponding column in the ``communityAggregate" matrix. When each node is in its own estimated community, the aggregate degree of the community is the same as the degree of the node. In each ``round," the algorithm orders the estimated communities by their aggregate degrees at the beginning of the round from largest to smallest. Then the algorithm visits these estimated communities in order, and for each estimated community $\hat{i}$, selects a candidate estimated community $\hat{j} \ne \hat{i}$ which maximizes
$$\mathcal{C}_{\hat{i}\hat{j}} = corr\left\{\big(\sum\limits_{u:u \in \hat{i}}W_{uv}, \sum\limits_{q:q \in \hat{j}}W_{qv}\big): v \in V\right\},$$
where $\sum_{u:u \in \hat{i}}W_{uv},  v \in V$, is a length $n$ vector where each entry has the total edge weight connecting nodes in estimated community $\hat{i}$ with each node in the network. Estimated communities $\hat{i}$ and $\hat{j}$ are merged if doing so does not decrease the measure $L$. If $\hat{i}$ and $\hat{j}$ are merged, and if $\hat{j}$ has not already been visited by the algorithm this round, the algorithm will visit the combined $\hat{i}$ and $\hat{j}$ when it would have visited $\hat{j}$. Each round completes after the algorithm has completed all scheduled visits. 

If at least two estimated communities have been merged, the algorithm proceeds to the next round. If no communities have been merged, the algorithm does a sweep, calculating $\mathcal{C}_{\hat{i}\hat{j}}$ for all estimated communities, and attempts to merge pairs of estimated communities in decreasing order of $\mathcal{C}$ values. As soon as any pair of estimated communities are merged, the algorithm stops the sweep and proceeds to the next round.  Algorithm \ref{alg:community-greedy} terminates when it goes through a full round and a sweep without merging any estimated communities, or when the number of estimated communities reaches 1.        

Two details of the algorithm should be explained further. First, the algorithm does not require merging communities to increase $L$ because at the outset, when all nodes are in their own estimated community, combining two communities \textit{cannot} increase $L$. Therefore, requiring a merger to increase $L$ would prevent the algorithm from gaining any traction. Second, why go through each round instead of constantly sweeping, or better yet just combining the estimated communities which would most increase $L$? In addition to this proposal being computationally costly, it may also lead to the initial formation of a single large but overly heterogeneous community since communities can only contribute to $L$ once they contain three or more nodes. As this is already a greedy algorithm, we view our design as a means of not overlooking any estimated community, and as a conservative precaution against premature optimization. 

\begin{algorithm}
\SetAlgoLined
\DontPrintSemicolon
\SetKwFunction{attemptMerge}{attemptMerge} 
\SetKwFunction{sweep}{sweep} 
\KwResult{$\{\, \widehat{i} \,\}$}
\KwInput{$W$}
\nl labels$^{(0)} = \vec{0}_n$\;
\nl labels$^{(1)} = \{1, ..., n\}$\; 
\nl $q$=1\;
\nl communityAggregate = $W$\;
\nl \While {$labels^{(q)} \ne labels^{(q-1)}$} {
	\nl $q$++\;
	\nl labels$^{(q)}$ = labels$^{(q-1)}$\;
	\nl labelOrder = sort labels by decreasing column sum of communityAggregate\; 
	\nl \While {$length(labelOrder) > 0$}{
		\nl labels$^{(q)}$, communityAggregate, labelOrder = \attemptMerge{communityAggregate, labels$^{(q)}$, labelOrder}\;
		\nl dequeue(labelOrder(1))\;
	}
	\nl \If{$labels^{(q)} = labels^{(q-1)}$}{
		\nl labels$^{(q)}$, communityAggregate = \sweep{communityAggregate, labels$^{(q)}$}\;
	}
}
\nl finalClustering = labels$^{(q)}$\;
\;
  \SetKwProg{myproc}{Procedure}{}{}
  \myproc{\attemptMerge{\text{communityAggregate}, \text{labels}$^{\text{(q)}}$, \text{labelOrder}}}{
		\nl newLabels = labels$^{(q)}$\;
		\nl newOrder = labelOrder\;
		\nl A = newOrder(1)\;  
		\nl B = $\argmax{\text{B}' \ne \text{A}}$ $corr([\text{communityAggregate(A)}], [\text{communityAggregate(B')}])$\;
 \tcp{[\text{communityAggregate(A)}] represents the Ath column in the communityAggregate matrix}
		\nl mergedLabels = newLabels\;
		\nl mergedLabels(mergedLabels = B) = A\;
		\nl mergedOrder = newOrder\; 
 		\nl mergedOrder(mergedOrder = B) = A\;
		\nl \If{calculateMeasureL(\text{W}, \text{mergedLabels}) $\ge$ calculateMeasureL(\text{W}, \text{newLabels})}{
			\nl newLabels = mergedLabels\;
			\nl newOrder = mergedOrder\;
			\nl [communityAggregate(A)] += [communityAggregate(B)]\;
			\nl Remove [communityAggregate(B)] column from communityAggregate\;
			}
		  \nl \KwRet newLabels, communityAggregate, newOrder
}
\caption{Greedy algorithm using $L$}
\label{alg:community-greedy}
\end{algorithm}

\begin{algorithm}
\DontPrintSemicolon 
\SetAlgoLined
\setcounter{AlgoLine}{28}
	\SetKwProg{myproc}{Procedure}{}{}
	\myproc{\sweep{communityAggregate, labels$^{(q)}$}}{
		\nl correlationOrder = sort non-matching label pairs by decreasing correlations of columns in communityAggregate\;
		\nl \While{$length(correlationOrder) > 0$}{
			\nl newLabels = labels$^{(q)}$\;
			\nl A = correlationOrder(1, 1)\;  
			\nl B = correlationOrder(1, 2)\;  
			\nl mergedLabels = newLabels\; 
			\nl mergedLabels(mergedLabels = B) = A\;
			\nl \eIf{calculateMeasureL(W, mergedLabels) $\ge$ calculateMeasureL(W, newLabels)}{
				\nl newLabels = mergedLabels\;
				\nl [communityAggregate(A)] += [communityAggregate(B)]\;
				\nl Remove [communityAggregate(B)] column from communityAggregate\;
				\nl dequeueAll(correlationOrder)
			}{
			\nl dequeue(correlationOrder(1))
			}
		}
\nl \KwRet newLabels, communityAggregate
}
\end{algorithm}

\subsection{Spatial clustering}\label{s:spatial}
For another perspective on community detection, we represent the original network as a matrix, and think of each row vector as a point in $n$-dimensional space corresponding to a particular node, where $n$ is the number of nodes in the network. Spatial clustering of the points then gives us communities, which is motivated as follows. Consider two nodes $u$ and $v$ in the same community with similar $\Psi$ values. For any third node $x \in j$, as $u \in i$ and $v \in i$, $H_{ij}(\Psi_u, \Psi_x)$ should be similar to $H_{ij}(\Psi_v, \Psi_x)$, thus $W_{ux}$ should be close to $W_{vx}$. Therefore, nodes in the same community with similar $\Psi$ values are expected to have similar edge weights, so entries in their corresponding rows should be similar in $n-2$ dimensions (the exceptions being due to no self loops). Therefore, in this $n$-dimensional space, two nodes with the same community and similar $\Psi$ values should be neighbors. 

Spatial clustering can be achieved using many existing clustering algorithms. For certain methods, such as those based on distances, one may need to account for the zeros on the diagonal by only calculating the distance between two nodes on the remaining $n-2$ dimensions. However, as with other spatial clustering problems, there is no algorithm which correctly clusters the nodes for every possible network.

\subsection{Spectral clustering algorithm incorporating $L$}\label{s:spectral-alg}
Appendix \ref{s:spatial} justifies the application of spatial clustering techniques for our networks of interest. In this section, we focus on one spatial clustering algorithm, spectral clustering, which we employ in tandem with the measure $L$ in (\ref{e:measure}), used to choose both a particular number of estimated communities as well as the best clustering for that number of communities. If the number of communities $K$ is known, we can follow the methodology of \cite{ng_jordan_weiss}. Measure the distance between every column in our network using a radial basis function kernel. Then construct a neighbors graph based on these distances, and a Laplacian based on this neighbors graph.  Finally, run K-means on the Laplacian to get clusterings. This can all be done using the existing \code{specc} function from the \proglang{R} package  \pkg{kernlab} just by specifying the number of centers, as \code{specc} will automatically select a scale parameter for the kernel. However, as this is not a deterministic algorithm, it can be helpful to run several replicates and take clustering that maximizes $L$.    

The remaining issue is how to choose the number of clusters, $K$, which is a priori unknown. However, we can use the introduced measure $L$ as a measure of clustering success, so we can impose a simple stopping rule, which is shown in Algorithm \ref{alg:community}. 


\begin{algorithm}
\SetAlgoLined
\KwResult{$\{\, \widehat{i} \,\}$}
\KwInput{$W$, replicates} \tcp{replicates is the number of times to run \code{specc} for each value of K}
\nl $\Delta$ = 1; $K$ =0\; 
	\nl \While {$\Delta > 0$}{
	\nl $K = K++$\;
	\nl clusterGoodness($K$) = $-\infty$\;
	\nl index $= 0$\;
    \nl \While{$\text{index}$ $<$ $\text{replicates}$}{
		\nl index$++$\;
		\nl candidateClustering = \code{specc}($W$, centers = $K$)\;
		\nl candidateGoodness =calculateMeasureL($W$, candidateClustering)\; 
			\nl \If{$candidateGoodness > clusterGoodness(K)$}{
    			\nl clusterGoodness($K$) = candidateGoodness\;
				\nl clustering($K$) = candidateClustering\;
			}
	}
\nl $\Delta$ = clusterGoodness($K$) $-$ clusterGoodness($K-1$) 
}
\nl finalClustering = clustering($K-1$)
 \caption{Community detection using spectral clustering and stopping criterion}
\label{alg:community}
 \end{algorithm}

\subsection{Clustering using normalized network}\label{s:rotated}

Another approach to community detection which seems to return interesting results is as follows: take the network as a whole and normalize each row in the original matrix by taking $N_{uv} = \frac{A_{uv}- \bar{A}_{u\bullet}}{SD(A_{u\bullet})}$, where $A_{u\bullet}$ represents the row in $A$ corresponding to node $u$. Next, calculate the eigenvalues $\lambda_1, \ldots, \lambda_n$ of the normalized network $N$, and note the index  corresponding to the largest difference in the absolute values of the real parts of these successive eigenvalues, $\argmax {1\le i \le n-1} |Re(\lambda_{i})| -  |Re(\lambda_{i+1})|$. Keep only the real parts of the first several eigenvectors corresponding to eigenvalues $\lambda_1, \ldots, \lambda_i$, and look for clusters in this lower dimensional space. 

A normalized version of the network presented in Figure \ref{f:reconstruction} is shown in the left plot of Figure \ref{f:normalized}. In the right plot of Figure \ref{f:normalized}, each point plots a row of the real parts of the first 2 eigenvectors of this row-normalized network, where the color of the point represents the true community of the corresponding node.

Plotting each row of these eigenvalues appears to give clearly separate clusters. Normalizing the network so all nodes have degree 1 doesn't give the same results. The first eigenvector of the original matrix often represents degree information, so the normalization to calculate $N$ should discard that degree information, but in the process, it also seems have some kind of downstream effect on other eigenvectors which helps them to capture the underlying communities. 
\begin{figure}
\begin{center}
    \includegraphics[width = 5in]{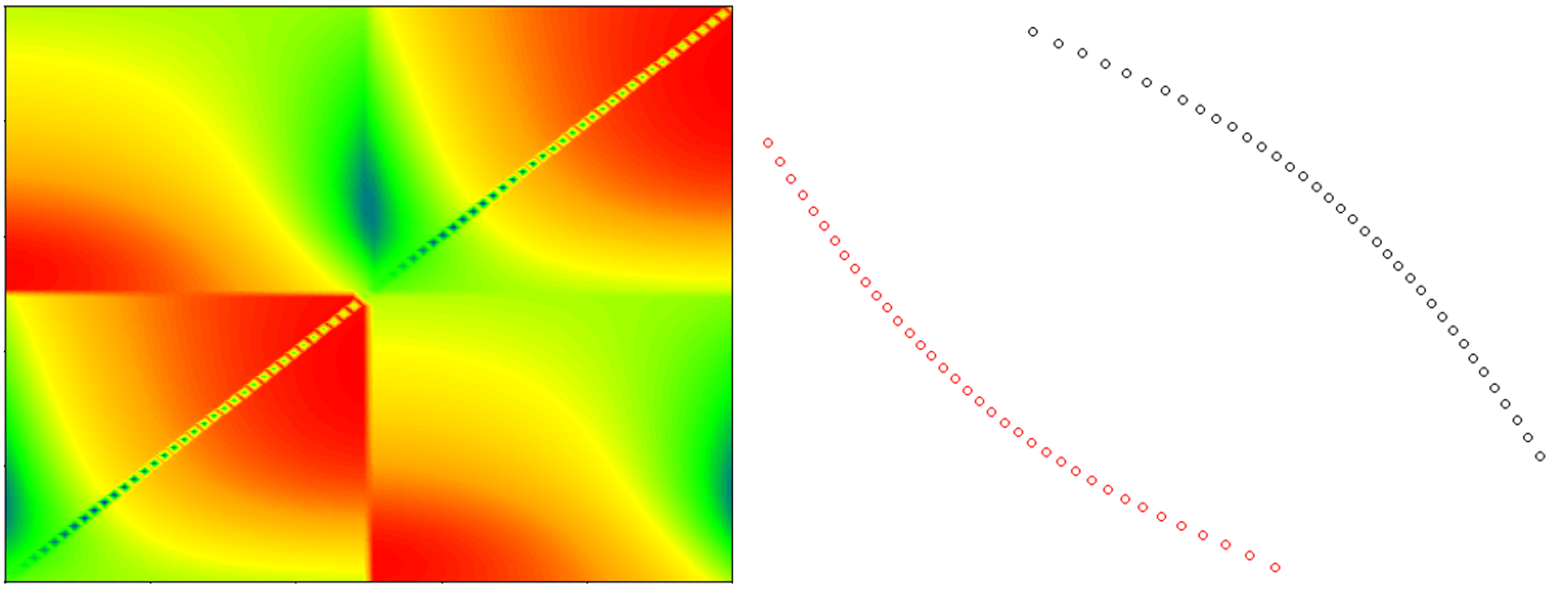}
\end{center}
\caption{Normalized network with implied clustering.}
\label{f:normalized}
\end{figure}
This phenomenon is not  specific to the constructed network in Figure \ref{f:reconstruction}. Figure \ref{f:rot-eig} shows a network generated via an NSM with an $H$-function combining two exponential random variables, in the style of the third plot in Figure \ref{f:different-h-funcs}, where within community connections show positive association and between community connections show Simpson association, along with the eigenvectors from the normalized version of that network. Again, the real parts of the eigenvectors of the normalized network show a clear separating plane between the two communities, but it is not along one of the 2 dimensions, but rather along a combination of them.
\begin{figure}
\begin{center}
    \includegraphics[width = 4in]{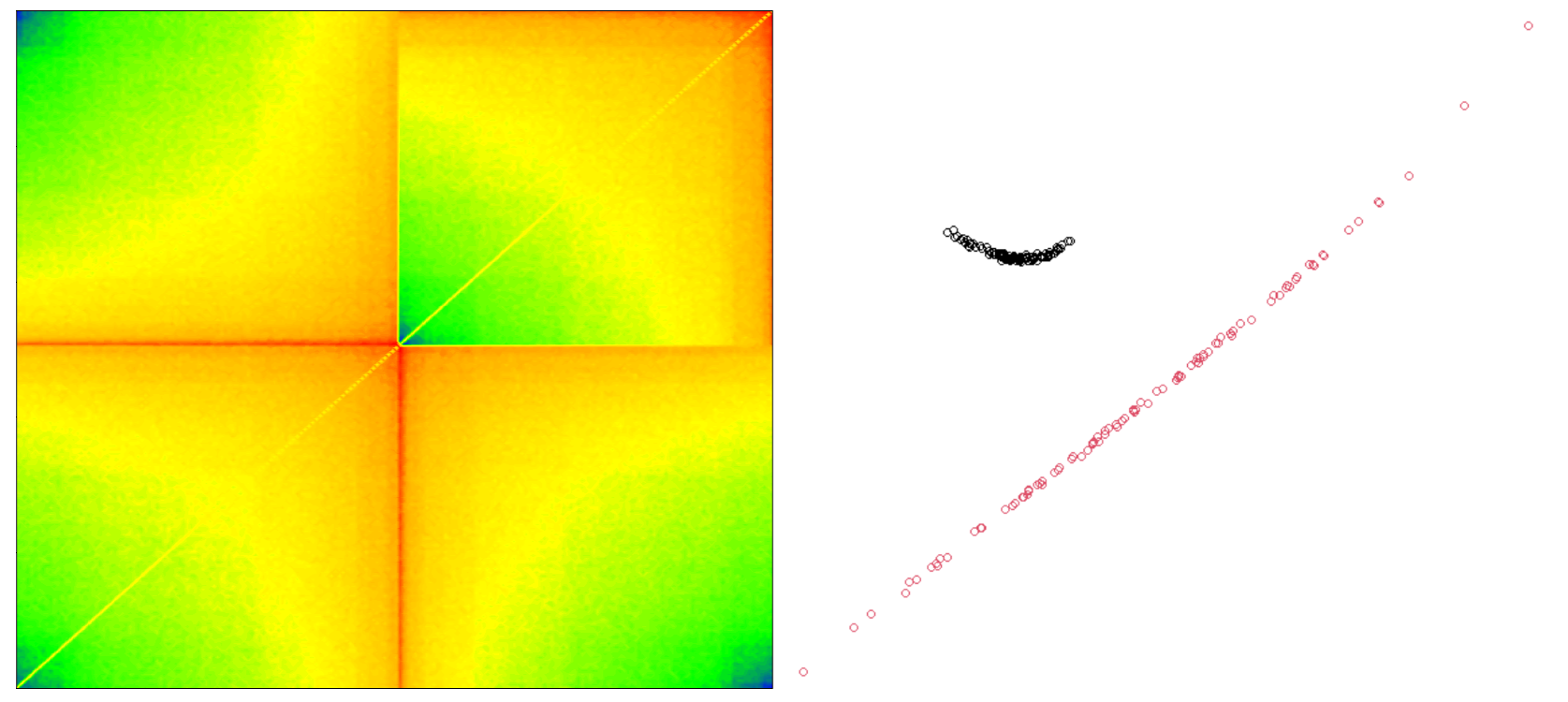}
\end{center}
\caption{Network generated via NSM, with the clustering from its normalized version.}
\label{f:rot-eig}
\end{figure}

There is some intuition for why clustering based on normalizing the network might work. Ignoring the diagonal, normalizing the matrix makes each column sum to 0, and have variance 1. Since this is a dense weighted network, there is no true concept of a “hub,” but instead just a node that has greater edge weights. For this reason, we have no reason to value one node/dimension over another, and we can avoid having to find complicated kernels with different variances across different dimensions. If we don't normalize, the distance in only one dimension can totally dominate. In the case where all associations are positive, when doing community detection, we want to ignore degree in favor of preference, so it makes sense to divide by a standard deviation. In the case where some associations are negative (or Simpson), extreme $\Psi$ values will indicate greater variance, but again, we still want to ignore this for the purposes of community detection, so it still makes sense to normalize. The reason to normalize is since each node is connecting to all (or most) others, we want to ensure that we are accounting for a given node's average edge weight and the variance of the edge weights. In this way, we are still looking for patterns of preferences across other nodes. By normalizing the columns, we can compare them to one another on an apples-to-apples basis. A negative weight in the normalized matrix means that the edge between the reference node and another node is less than the average weight emanating from the reference node. For the purposes of community detection in this model, this is really what we care about, that is, patterns across nodes in a given community that show preferences across nodes throughout the whole network.

These last points also indicate the zeros on diagonal can bring up issues.  Imagine a situation where every other edge weight emanating from a node is extremely large but low variance. The 0 will shrink the average and increase the standard deviation by a lot, but it's purely artificial. If we use a distance between nodes/columns to cluster, it is imperative to ignore both the rows corresponding to those nodes. If there is actually a relatively large edge weight between them, that edge weight will be subtracted and squared twice. This is also purely artificial due to no self loops.

\section{Additional simulations}\label{more-sims}
\subsection{$H$-Normal LSM with $\sigma =0$} 
Figure \ref{f:no-noise-LSM} shows an $H$-Normal LSM with $\sigma =0$, where the value for within community edges are equal to $5+ 3\Phi_1^{-1}(\Psi_u) + 3\Phi_1^{-1}(\Psi_v)$. The between community edges, assuming $u$ is in community 1 and $v$ is in community 2, are equal to $8 - 3\Phi_1^{-1}(\Psi_u) + 1.5\Phi_1^{-1}(\Psi_v)$. The original network is shown on the left. In the middle is the reconstructed estimated network, which looks nearly identical to the original network. $\widehat{\sigma}$ is essentially 0 for this network, so the MSE for each subnetwork is shown on the right. 
 \begin{figure}[h!]
 \includegraphics[width=\linewidth]{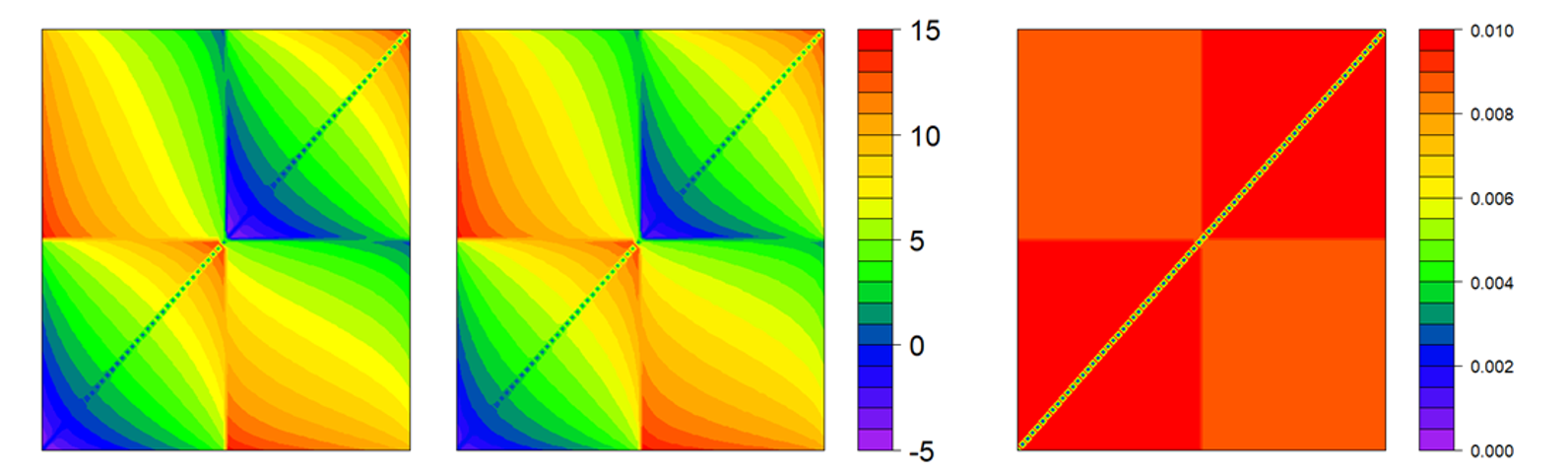}
\caption{$H$-Normal LSM with $\sigma =0$, its estimate, and the subnetwork level MSE.}
\label{f:no-noise-LSM}
\end{figure}

\subsection{Varied network}
Figure \ref{f:tricky} adds more complexity and departs even further from an $H$-Normal NSM, this time with 200 nodes and four communities of possibly different sizes. In this case, nodes are not assigned $\Psi$ values but i.i.d. Gamma(shape = 5, scale = 10) random values $\Gamma_u^{(1)}$ and $\Gamma_u^{(2)} = 1/\Gamma_u^{(1)}$.  However, within each community, $\Gamma_u^{(1)}$ and $\Gamma_u^{(2)}$ are normalized by dividing by $\sum_{u \in i} \Gamma_u^{(1)}$ and $\sum_{u \in i} \Gamma_u^{(2)}$, respectively. Each pair of communities is also assigned independently a random Gamma(shape =50$\sqrt{10}$, scale = 50$\sqrt{10}$) value $\Gamma_{ij}$. There is no mathematical significance to these parameters, they were chosen only to create a striking image. Within community edges are distributed according to a negative binomial distribution where the target number of successful trials is $2\times\Gamma_{ii}$ and the probability of success in each trial is $(1-\Gamma_u^{(1)})(1- \Gamma_v^{(1)})$. ``Adjacent" communities in the graph are distributed according to a negative binomial where the target number of successful trials is $1.5 \times\Gamma_{ij}$ and the probability of success in each trial is $(1-\Gamma_u^{(2)})(1- \Gamma_v^{(2)})$. Connections between communities 1 and 3 or between communities 2 and 4 are distributed according to a Poisson distribution with parameter $100\times\Gamma_{ij}\Gamma_u^{(1)}\Gamma_v^{(1)}$. Finally, connections between 1 and 4 are normally distributed with mean $10000\times(\Gamma_u^{(1)}+\Gamma_v^{(1)})$ and variance 1. Though the original network looks noisy, the estimate seems to capture a smooth approximation. In fact, in this network, $\hat{\sigma}$ is indistinguishable from 0 everywhere, so the MSE of each subnetwork, shown in the right plot of Figure \ref{f:tricky}, would be used for the bootstrap. 
\begin{figure}[h!]
 \includegraphics[width=\linewidth]{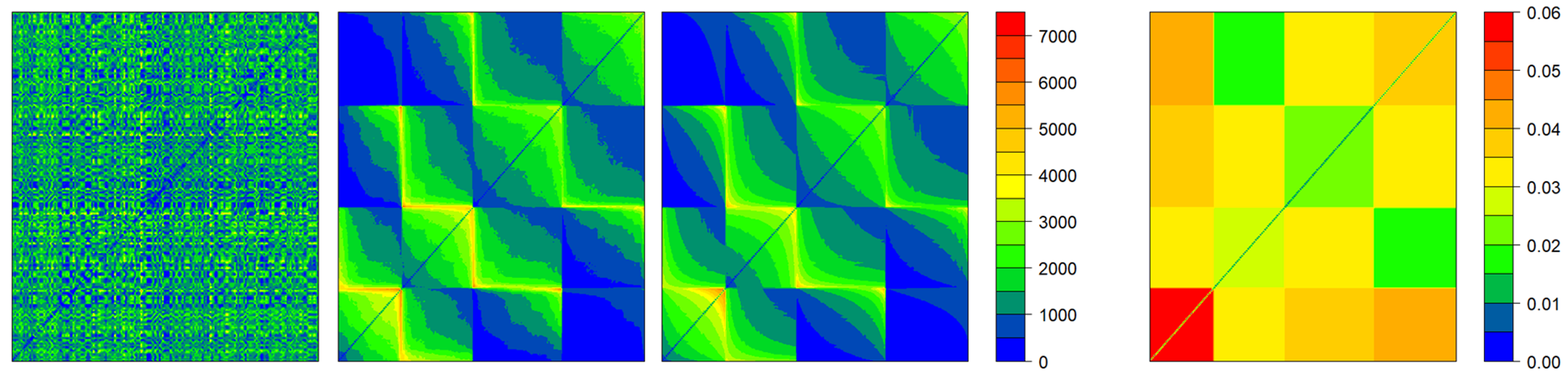}
\caption{From left to right: original, reordered, estimate, MSE of each subnetwork.}
\label{f:tricky}
\end{figure}

\subsection{Correlation matrix}
In this network, we first generate 3 ``lodestar" series $\mathcal{L}_1, \mathcal{L}_2, \mathcal{L}_3$ of length 1000 composed entirely of i.i.d. U$(0,1)$ random variables. We create a network with 200 nodes where the first 50 nodes each get a series which are (to different degrees) positively correlated with the first lodestar series, and negatively correlated with the second lodestar series. This is done by calculating the series for node $u$ at time $t$ as $$u(t) = \beta_1(u)\mathcal{L}_1(t) +  \beta_2(u)\mathcal{L}_2(t) +\beta_3(u)\mathcal{L}_3(t) + \big(1-\sqrt{\beta_1^2(u)+\beta_2^2(u) +\beta_3^2(u)}\big)\epsilon_u(t),$$ where $\epsilon_u(t) \sim$ U$(0, .35)$ is an idiosyncratic term for each node at each time step. For the first 50 nodes, the $\beta_1(u)$ values are positive and increasing with $u$, while the $\beta_2(u)$ values are negative and getting more negative with $u$, and $\beta_3(u) =0$. For the second 50 nodes, the $\beta_2(u)$ values are positive and increasing with $u$, while the $\beta_1(u)$ values are negative and getting more negative with $u$. The 101st to 150th nodes get series which are increasingly positively correlated to $\mathcal{L}_3$ series and increasingly negatively correlated to $\mathcal{L}_2$. Finally, the last 50 nodes have series which are increasingly negatively correlated to $\mathcal{L}_3$. Taking the correlations of the 200 series, we get the correlation matrix on the left of Figure \ref{f:corr-less}. Though the subnetwork between the first and fourth communities may look disordered, because nodes within each community are relatively correlated with each other, there is a discernible ordering in that between community subnetwork. Even so, from a practical perspective, edge weights only take on a narrow range of values near 0 in that subnetwork.  
\begin{figure}[h!]
 \includegraphics[width=\linewidth]{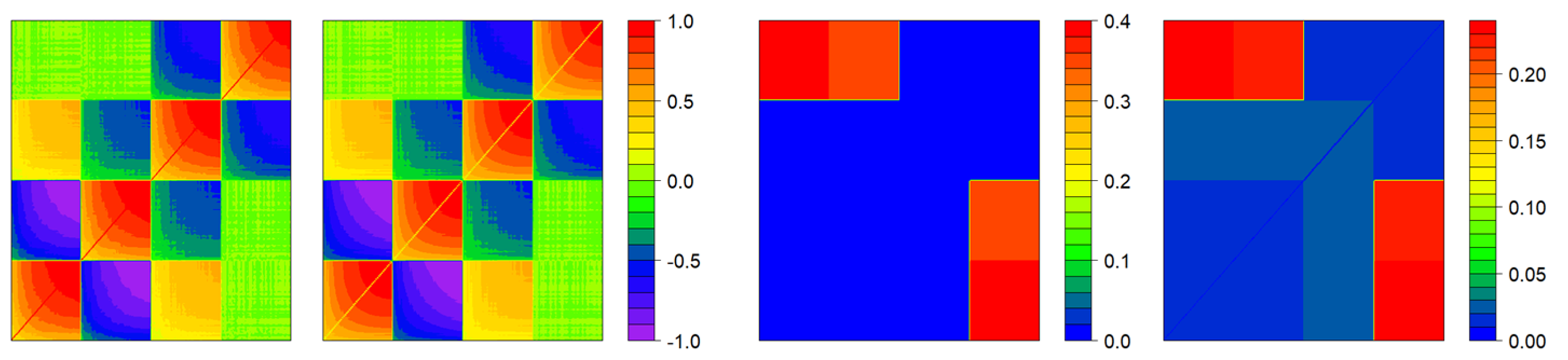}
\caption{Correlation network. From left to right: original, estimate, $\hat{\sigma}$ for each subnetwork, MSE of each subnetwork.}
\label{f:corr-less}
\end{figure}

\subsection{Edge weights with injected noise}\label{s:noisyedgesim}
In plot A of Figure \ref{f:noisyedgesim}, we start with the same underlying network as in Figure \ref{f:no-noise-NSM} but add Gaussian noise $\zeta_{uv}$ centered at 0 with a variance of 36 \textit{to the final edge weights}, not in ``normal" space, so the edge weights in the network can go below 0 and above 150. In plot B, the added Gaussian noise $\zeta_{uv}$ has variance 100 for within community edges and variance 225 for between community edges. In both cases, the reconstructed networks recover the underlying pattern, but the noisier network is estimated more coarsely.  
\begin{figure}[h!]
 \includegraphics[width=\linewidth]{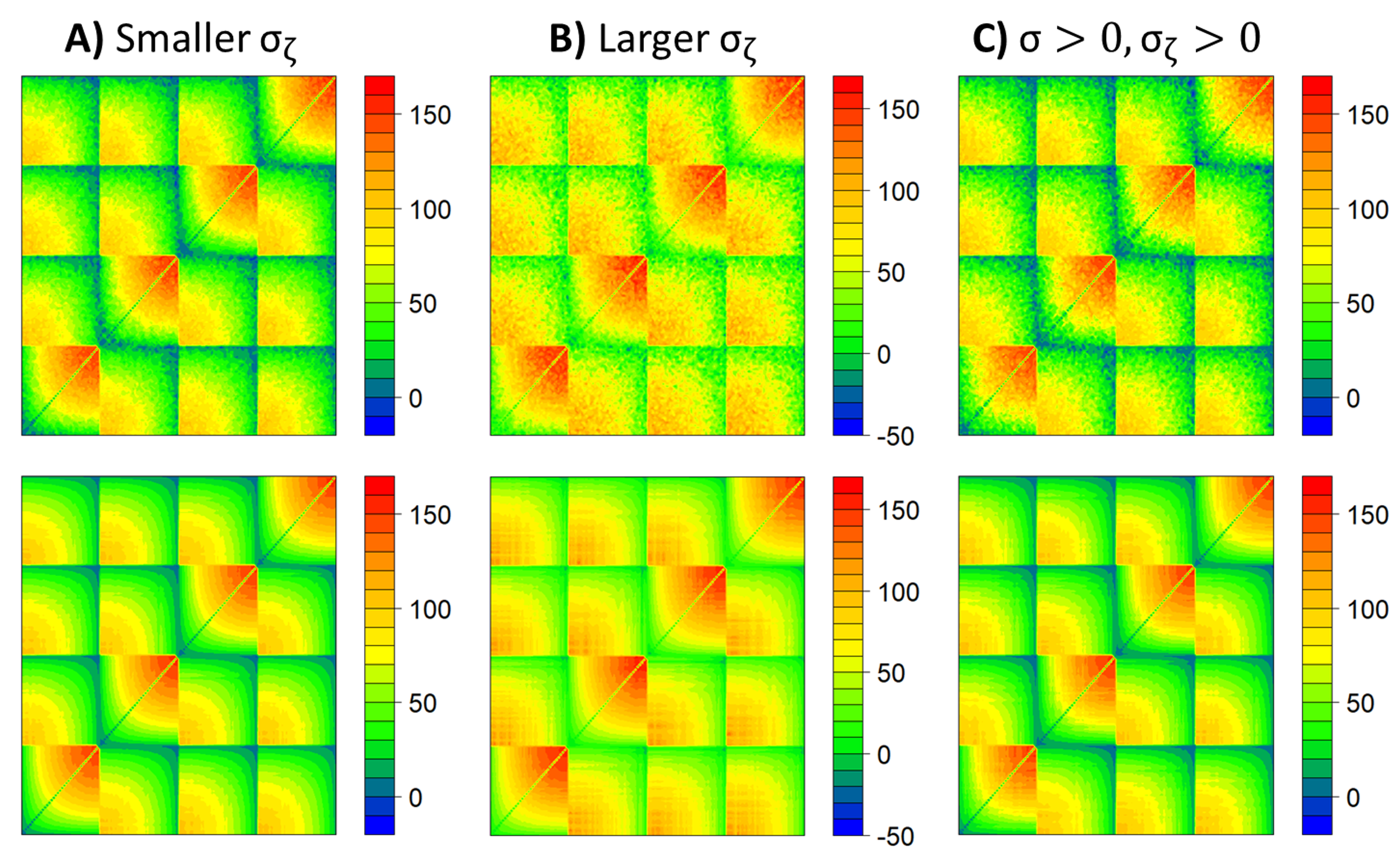}
\caption{The network in plot A of Figure \ref{f:no-noise-NSM} with different values for $\sigma_{\zeta}^2$}
\label{f:noisyedgesim}
\end{figure}

Finally, in plot F, the same network is taken with $\sigma = .05$ everywhere and external noise is included by adding $\zeta_{uv}$ with variance 36 to the final edge weights. Even in this last case, the underlying signal is broadly recovered.    

\subsection{Multiplying sociability parameters}
It may be worth considering alternative models which can generate networks similar to LSMs and NSMs. For example, if edge weights are generated by multiplying sociability parameters, some surprising things happen. Ignoring diagonal zeros, if we take the values from .01 to 1 by .01 on each axis and let the value in the matrix equal the product of the axes, that results in the leftmost plot in Figure \ref{f:mult}.  In other words, if all values are positive, the result looks like the 4th plot in Figure \ref{f:different-h-funcs}. But if we center each axis to be mean 0 and recalculate, that produces the second plot in Figure \ref{f:mult}. This is not the intended setting for the models discussed in this paper, but if we were to treat these networks as such, in the first case, every node would be put into a single community. In the second case, those nodes with negative values would be put in a separate community from nodes with positive values. Using this split and reordering based on within community degree gives the rightmost plot in Figure \ref{f:mult}. 

In principle, were noise added to the edge weights of this network, knowing the true generating model type might improve estimation, as one may be able to smooth noise out over more observations by keeping all nodes in one large community. However, even using the ``wrong" communities, our estimation procedure appears to replicate the underlying network. Even though the generating process for this network is the same across both estimated communities, separating the second network into two communities is therefore a reasonable choice, especially since the two communities can be so easily defined. This kind of pattern only arises when multiplying nodes with positive sociabilities and others with negative sociabilities, not when all sociabilities have the same sign. While multiplying sociabilities hints at the idea of negative association, networks generated from these models are still restricted to symmetric contours of the type seen in the 4th plot of Figure \ref{f:different-h-funcs}. However, simply multiplying sociability parameters can’t generate networks that have contours of the type seen in the first 3 plots of figure \ref{f:different-h-funcs}, nor can it give positive association patterns within a community but negative association patterns between communities when there are more than two communities.    
\begin{figure}
\begin{center}
 \includegraphics[width = \linewidth]{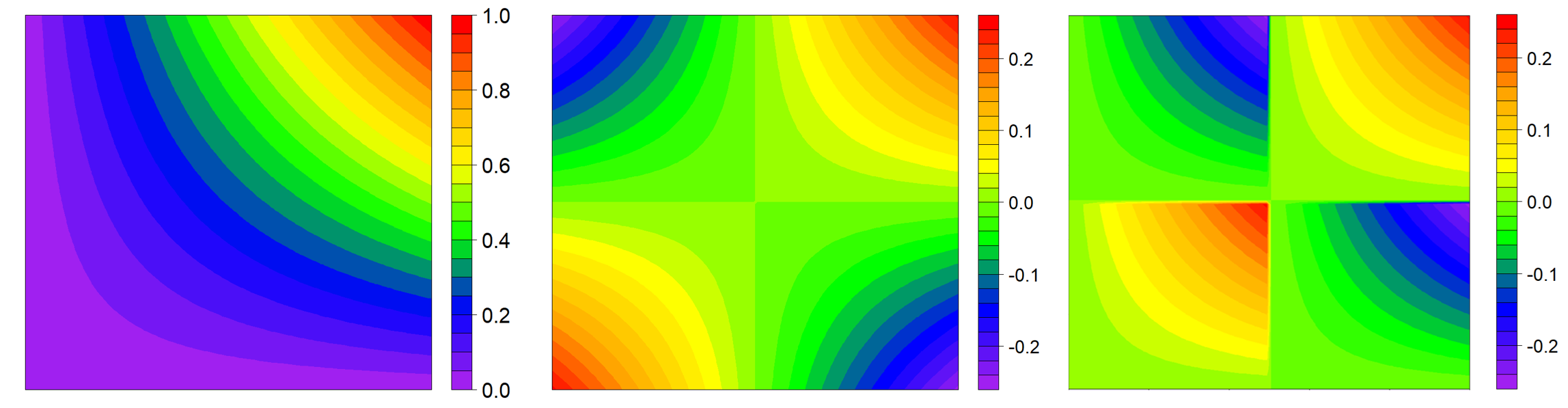}
\end{center}
\caption{From left to right: Network generated by multiplying sociability parameters. Network generated by demeaning sociability parameters then multiplying. Reordering of second network based on within estimated community degree.}
\label{f:mult}
\end{figure}   

\section{Higher dimensional $H$-functions and ``failure"}\label{s:Hext}
Thus far, we have defined $H$-functions as a class of functions that take in two uniform random variables and output another uniform random variable. This model can be extended to include a broader class of functions.   

\begin{definition}($d$-dimensional $H$-function.) A function $H:(0, 1)^d\rightarrow (0,1)$ is a $d$-dimensional $H$-function if the inputs are $d$ uniform random variables, and the output is a uniform random variable which is monotonic in each argument. 
\end{definition}

Note the $H$-Normal NSM (\ref{e:HNSM}) can be expressed in terms of a 3-dimensional $H$-function as:
\begin{equation}\label{e:HNSM2}  
W_{u v} = G^{-1}\left (H(\Psi_{u}, \Psi_{v}, \eta_{u v})\right)
\end{equation}
with $\eta_{uv} = \Phi_1(\epsilon_{uv}) \sim U(0,1)$, where
\begin{equation}\label{e:3DH} 
H(x, y, \eta) =\Phi_1\left(\frac{1}{\sqrt{1+\sigma^2}} \Phi_1^{-1}(h(x,y)) + \frac{\sigma}{\sqrt{1+\sigma^2}} \Phi_1^{-1}(\eta)\right),
\end{equation}
where $h(x,y)$ is a 2-dimensional $H$-function. The first two plots in Figure \ref{f:noise-failure} show the values along the $x$ and $y$ axes of functions which have the form of (\ref{e:3DH}) where $h(x,y)$ is in the schema of (\ref{e:HCDF}) such that $F_1$ and $F_2$ are Exponential distributions, and $F_{1, 2}$ is a Gamma distribution. As ``errors" are injected at each ($x$, $y$) pair, this can also be seen as a subnetwork generated via an $H$-Normal NSM where $H(x,y)$ in (\ref{e:HNSM}) is that of the third plot in Figure \ref{f:different-h-funcs}, and $\eta$ is injected at the edge level.

\begin{figure}
 \includegraphics[width=\linewidth]{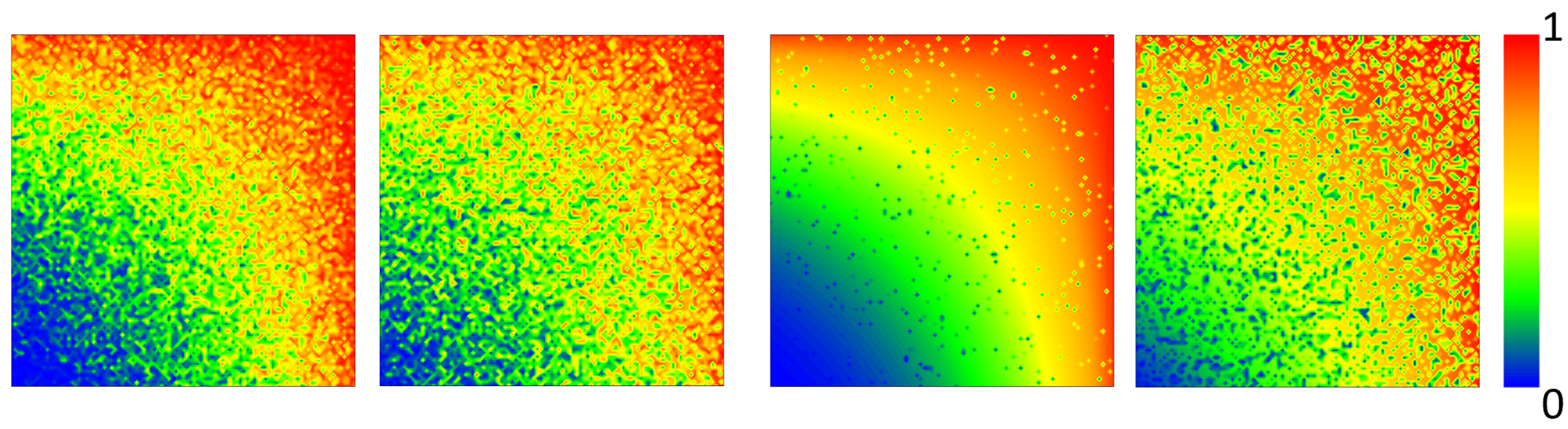}
\caption{Discretized 3-dimensional $H$-functions where axes show $x$ and $y$ but $\eta$ is drawn randomly at each ($x$, $y$) pair. The 2-dimensional $H$-function with arguments $x$ and $y$ is the same as in the third plot of Figure \ref{f:different-h-funcs}. From left to right: $H$-Normal NSM as in (\ref{e:HNSM}) with $\sigma = .05$. $H$-Normal NSM with $\sigma=.5$. Failure with $\alpha = .95$.  Failure with $\alpha = .5$.}
\label{f:noise-failure}
\end{figure}
The specific $H$-Normal NSM (\ref{e:H3-with-noise}) is a simple example which can be expressed in closed form as a 3-dimensional $H$-function as in (\ref{e:3DH}) where
\begin{equation}\label{e:3DH2}
H(x, y, \eta) =\Phi_1\left(\frac{1}{\sqrt{1+\sigma^2} \sqrt{1+\rho^2}} \Phi_1^{-1}(x) + \frac{\rho}{\sqrt{1+\sigma^2 }\sqrt{1+\rho^2}} \Phi_1^{-1}(y) + \frac{\sigma}{\sqrt{1+\sigma^2}} \Phi_1^{-1}(\eta)\right).
\end{equation}
In the network context, $\rho$ defines the relative influence of each of the node sociabilities, while  $\sigma$ controls the ``signal-to-noise" ratio of this 3-dimensional $H$-function. If one imagines observing several instances of the same network given by (\ref{e:HNSM2}) with idiosyncratic $\eta$ values, then increasing $\sigma$ would increase the variance of the individual edge weights from one instance to another. 
The 3-dimensional $H$-function (\ref{e:3DH2}) can be viewed as a composition of two 2-dimensional $H$-functions as follows:
$$ H(x, y, \eta) = H_{\sigma^2}(H_{\rho^2}(x,y), \eta), $$
where $H_{\rho^2}$ and $H_{\sigma^2}$ are both of the form (\ref{e:H-normal-rho}) but with different variance parameters, as indicated by their subscripts. 
All the observations about (\ref{e:3DH2}) do not depend on using 3-dimensional $H$-functions built from normal distributions or even $H$-Normal NSMs, but rather one can use any $d$-dimensional functions similar to (\ref{e:HCDF}) where $d>2$, with suitable adjustments based on the chosen distributions $F_1, F_2$, and $F_{1,2}$. 

One general method of creating higher dimensional $H$-functions is by chaining together lower dimensional $H$-functions as follows:
\begin{equation}\label{H-chain} 
H(x, y, \eta) = F_{1,2}\left(  F_1^{-1} \left( F_{3,4}(F_3^{-1}(x), F_4^{-1}(y)) \right), F_2^{-1}(\eta)   \right).
\end{equation}  
In general, the inner functions $F_3, F_4$ do not need to share the same form as the outer functions $F_1, F_2$. In the left plot of Figure \ref{f:Exp-Unif}, $F_3$, $F_4$ and $F_{3,4}$ correspond to the third plot in Figure \ref{f:different-h-funcs}, $\eta \sim$ U$(0,1)$ for each edge, and $F_1$, $F_2$ and $F_{1,2}$ correspond to the rightmost plot in Figure \ref{f:different-h-funcs}. The right plot of Figure \ref{f:Exp-Unif} swaps the roles of the inner and outer 2-dimensional $H$-functions in the left plot of Figure \ref{f:Exp-Unif}. While building higher dimensional $H$-functions in this way provides a lot of flexibility, this method may not give simple closed form expressions, and may not guarantee identifiability.  

\begin{figure}
\begin{center}
\includegraphics[width=0.75\textwidth]{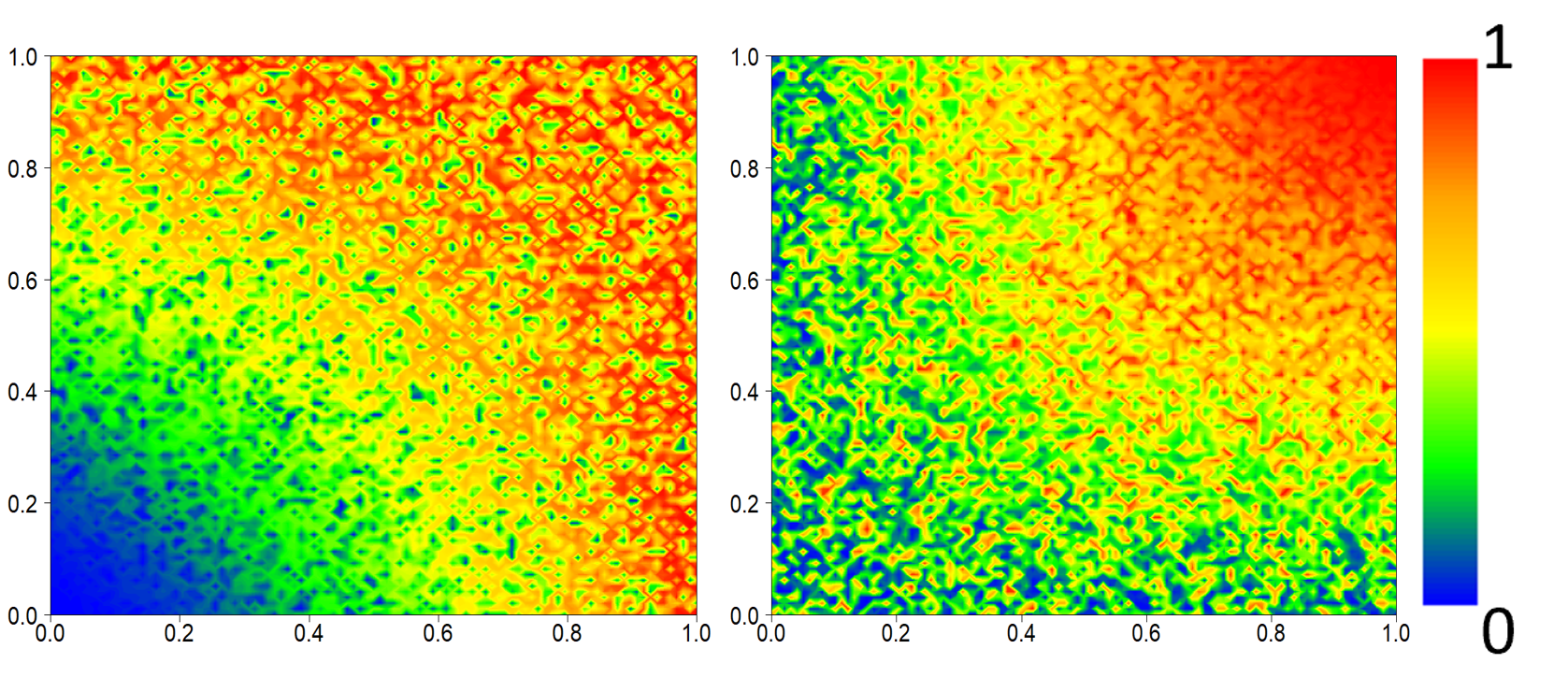}
\end{center}
\caption{3-dimensional $H$-functions of the form (\ref{H-chain}), where axes represent values of $x$ and $y$.}
\label{f:Exp-Unif}
\end{figure}

While the $H$-Normal NSM is a model with additive ``error," (\ref{e:HNSM2}) is more general. The following can be used to generate a different kind of ``error," one we shall call \textit{failure}. In this context, let 

\begin{equation}\label{e:failure}
H(x, y, \eta) = ( h(x, y))^\alpha \delta^{1-\alpha},
\end{equation}
where $\alpha$ is a value between 0 and 1, $h$ is a 2-dimensional $H$-function, and $\delta^{1-\alpha}$ is given by $F_\alpha^{-1}(\eta)$ where 
 $$F_\alpha(x) =
  \begin{cases}
 0, & x<0, \\
  \frac{x}{1-\alpha}, & 0 \le x \le 1-\alpha, \\
1, & x >  1-\alpha.
   \end{cases}
$$
Seen another way, 
$$\delta =
  \begin{cases}
 1 & \text{with probability } \alpha, \\
  \eta^{\frac{1}{1-\alpha}} & \text{with probability } 1-\alpha.
   \end{cases}
$$ 

One can see that (\ref{e:failure}) satisfies the definition of a 3-dimensional $H$-function by taking the Laplace-Stieltjes (LS) transform of the log of (\ref{e:failure}) with uniform inputs and recognizing that it matches the LS transform of the log of a uniform distribution. For $\eta \sim$ U$(0,1)$, the model (\ref{e:failure}) can also be written as
$$ H(x, y, \eta) \overset{d}{=}
  \begin{cases}
 h(x, y)^\alpha, & \text{with probability } \alpha, \\
   h(x, y)^\alpha \varepsilon, & \text{with probability } 1-\alpha,
   \end{cases}
$$ 
where $\varepsilon$ is also a uniform random variable. When $\alpha$ = 1, there is no ``error;" when $\alpha$ = 0, there is no degree correction, and increasing $\alpha$ increases the ``signal-to noise-ratio." As a contrast to the additive ``error" regime, in the failure case, when $\alpha$ is large, several different instances of the same network would share many of the \textit{exact} same edge weights, as in expectation, $100\times(1-\alpha)\%$ of the edge weights are given precisely by $h(x, y)^\alpha$.  The reason to call this kind of error ``failure" is that rather than defining a distribution that is concentrated near $h(x, y)$ with relatively small variation, even when $\alpha$ is large, there are infrequent occasions where the value will fall far below the modal value of $h(x, y)^\alpha$. The injected error will never raise the value greater than $h(x, y)^\alpha$, which also accounts for why the modal value lies at $h(x, y)^\alpha$ rather than $h(x, y)$. This kind of variation is reminiscent of each component in a system possessing a particular capacity, but occasional component failures cause that capacity to not be met. Figure \ref{f:noise-failure} depicts different levels of $\sigma$ and $\alpha$ being injected into the third plot of Figure \ref{f:different-h-funcs}. 


Consider the example of a road network, where the vertices are geographic locations, edges are roads, and edge weights are the number of cars that travel along the road each day. In this case, there should be degree correction, as there should be heavier traffic between certain locations than others. However, in addition to random variation for travelers along each road (which could be represented by additive ``error"), on some days, whether due to accidents, construction, or some other issue, the traffic along certain roads may fall dramatically. This latter case would be an example of failure. In this case, it may be better to model the network using a 4-dimensional $H$-function of the form
$$H(x,y, \eta_1, \eta_2) =   (F_{1,2,3}\left(F_1(x)+ F_2(y) + F_3(\eta_1)\right))^{\alpha} \delta_2^{1-\alpha},$$
which would incorporate additive ``error" through $\eta_1$, and failure through $\delta_2^{1-\alpha}$, which is a function of $\eta_2$. While this would be a valid 4-dimensional $H$-function in theory, in practice there may be issues arising from dependence between inputs, such as the probability of failure being correlated with the additive ``error."  Furthermore, one may want to include other covariates in the $H$-function which are not completely random, but rather systematic features, like time, that are different from node sociabilities, for use with tensors rather than matrices. There are many other kinds of $H$-functions that can incorporate various covariates and errors to model specific phenomena. Our goal, however, is not to catalog these possibilities, but to illustrate the richness and generality of the class of $H$-functions, particularly for generating random networks. 
\end{document}